\documentclass[10pt]{IEEEtran}
\usepackage[margin=2.1cm]{geometry}

\usepackage{amsfonts,amssymb,amsmath,amsthm,dsfont,bm}
\usepackage[normalem]{ulem}
\usepackage[footnotesize]{caption}
\usepackage{graphicx}
\usepackage[linesnumbered,ruled,lined,boxed,commentsnumbered]{algorithm2e} 
\usepackage{mathtools} 
\usepackage{subfig}
\usepackage{hyperref}
\usepackage{url}
\usepackage[dvipsnames,table]{xcolor}
\usepackage{cite}
\usepackage{multirow}
\usepackage{booktabs} 
\usepackage{tabularx}
\usepackage{enumitem}

\title{Asymmetric compressive learning guarantees\\with applications to quantized sketches}

\author{Vincent Schellekens\IEEEauthorrefmark{1}\thanks{\IEEEauthorrefmark{1} E-mail: {\em \{vincent.schellekens, ~laurent.jacques\}@uclouvain.be}. ISPGroup, INMA/ICTEAM, UCLouvain, Louvain-la-Neuve, Belgium. VS and LJ are funded by Belgian National Science Foundation
		(F.R.S.-FNRS).} and Laurent Jacques\IEEEauthorrefmark{1}}

\newcommand{\Rbb}{\mathbb{R}}

\newtheorem{theorem}{Theorem}
\newtheorem{definition}{Definition}
\newtheorem{proposition}{Proposition}
\newtheorem{corollary}{Corollary}
\newtheorem{lemma}{Lemma}

\newtheorem{assumption}{Assumption}

\newcommand{\supp}{{\rm supp}\,}

\newcommand{\sign}{{\rm sign}}

\newcommand{\ud}{\mathrm{d}}

\renewcommand{\leq}{\leqslant}
\renewcommand{\geq}{\geqslant}

\DeclareMathOperator{\iid}{iid}

\newcommand{\bb}{\mathbb}

\newcommand{\ts}{\textstyle}

\newcommand{\bs}{\boldsymbol}
\newcommand{\cl}{\mathcal}
\newcommand{\ie}{\emph{i.e.}, }
\newcommand{\eg}{\emph{e.g.}, }

\newcommand{\eqdef}{:=}

\renewcommand{\Vec}[1]{\bs{#1}} 
\newcommand{\distiid}{\sim_{i.i.d.}} 
\newcommand{\expec}[1]{\mathop{{}\mathbb{E}}_{#1}} 

\newcommand{\im}{\mathrm{i}\mkern1mu} 
\DeclarePairedDelimiterX{\norm}[1]{\lVert}{\rVert}{#1} 
\newcommand{\Integr}[4]{\int_{#1}^{#2}#3\mathrm{d}#4} 

\newcommand{\wt}{\widetilde}
\newcommand{\wh}{\widehat}

\newcommand{\ds}{\cl{X}} 
\newcommand{\sketchop}[1][]{\cl{A}_{#1}}
\newcommand{\disttrue}{\cl P_0}
\newcommand{\distemp}{\wh{\cl P}_{\ds}}
\newcommand{\loss}{\ell}
\newcommand{\risk}{\cl R}
\newcommand{\SSE}{\mathrm{SSE}}
\newcommand{\LL}{\mathrm{LL}}
\newcommand{\RFF}{\scriptscriptstyle\mathrm{RFF}}
\newcommand{\setdistributions}{\mathcal{M}}
\newcommand{\modelset}{\mathcal{G}}
\newcommand{\measuredset}{\widehat{\mathcal{G}}}
\newcommand{\soltruerisk}{\Vec{\theta}^{*}}
\newcommand{\solerm}{\wt{\Vec{\theta}}}
\newcommand{\solsketch}{\wh{\Vec{\theta}}}
\newcommand{\solquant}{\wh{\Vec{\theta}}'}
\newcommand{\LRIP}{\mathrm{LRIP}}
\newcommand{\LPD}{\mathrm{LPD}}
\newcommand{\fmod}{\mathrm{mod}} 
\newcommand{\smod}{\scriptscriptstyle\mathrm{mod}} 
\DeclareMathOperator{\tmod}{mod} 
\newcommand{\lripc}{\gamma}

\makeatletter
\newcommand{\RemoveAlgoNumber}{\renewcommand{\fnum@algocf}{\AlCapSty{\AlCapFnt\algorithmcfname}}}
\newcommand{\RevertAlgoNumber}{\algocf@resetfnum}
\makeatother

\newcommand{\br}[1]{\textcolor{red}{[\textbf{BR:#1}]}}

\renewcommand{\br}[1]{}

\begin{document}

\maketitle

\begin{abstract}
The compressive learning framework reduces the computational cost of training on large-scale datasets. In a sketching phase, the data is first compressed to a lightweight sketch vector, obtained by mapping the data samples through a well-chosen feature map, and averaging those contributions. In a learning phase, the desired model parameters are then extracted from this sketch by solving an optimization problem, which also involves a feature map. When the feature map is identical during the sketching and learning phases, formal statistical guarantees (excess risk bounds) have been proven.

However, the desirable properties of the feature map are different during sketching and learning (\eg quantized outputs, and differentiability, respectively). We thus study the relaxation where this map is allowed to be different for each phase. First, we prove that the existing guarantees carry over to this asymmetric scheme, up to a controlled error term, provided some Limited Projected Distortion (LPD) property holds. We then instantiate this framework to the setting of quantized sketches, by proving that the LPD indeed holds for binary sketch contributions. Finally, we further validate the approach with numerical simulations, including a large-scale application in audio event classification.
\end{abstract}

\begin{IEEEkeywords}
	Compressive statistical learning, asymmetric embeddings, quantization, Gaussian mixture modeling.
\end{IEEEkeywords}

\section{Introduction}
\label{sec:introduction}
The availability of large-scale datasets has risen tremendously over the past few years, which empowered machine learning solutions in a staggering amount of application areas. However, training from ever larger quantities of data requires ever more computational resources (\eg memory storage and processing time). For instance, datasets that cannot fit on the memory of standard workstations are increasingly common. In order to keep up with the accelerating rate at which data is produced, the design of novel resource-efficient learning paradigms has thus grown into an important research area~\cite{al2015efficient,qiu2016survey,l2017machine}.

\emph{Compressive (statistical) learning} (CL) was proposed to learn from datasets $\ds = \{ \bs x_i \in \bb R^d \}_{i=1}^n$ of massive scale (in particular, with a large number of examples $n$, typically at least several millions) while keeping computational resources (\ie memory, training time) under control~\cite{gribonval2017compressiveStatisticalLearning,gribonval2020sketching}. To do so, CL first compresses the data as a single $m$-dimensional (possibly complex-valued) vector $\Vec{z}_{\Phi,\ds} = \frac{1}{n} \sum_{i = 1}^{n} \Phi( \Vec{x}_i )$, called \emph{sketch}, by simple averaging of a \emph{feature map} $\Phi : \bb R^d \rightarrow \bb C^m$. Crucially, computing the sketch requires only one single pass over $\ds$, which is moreover easy to distribute and parallelize thanks to the independent contributions of each $\bs x_i$~\cite{gribonval2017compressiveStatisticalLearning}. This ensures that the sketch can always be constructed in an efficient manner, even when the dataset size grows drastically.

The choice of the feature map $\Phi$ determines the range of machine learning tasks that one is able to solve from the sketch $\Vec{z}_{\Phi,\ds}$ (note that $\Phi$ should be nonlinear, otherwise the sketch cannot capture anything beside the data mean).
For example, it is possible to solve k-means~\cite{keriven2016compressive} or Gaussian mixture modeling~\cite{keriven2016GMMestimation} within this framework, when $\Phi$ is chosen to be random Fourier features (RFF)~\cite{Rahimi2008RFF}. These features are defined as the complex exponential of random projections of the data, \ie 
\begin{equation}
\label{eq:RFF}
	\ts \Phi_{\RFF}(\Vec{x}) = \frac{1}{\sqrt{m}}\exp(\im (\bs\Omega^{\top}\Vec{x} + \Vec{\xi})),
\end{equation}
where the exponential is applied component-wise, and $\bs\Omega \in \bb R^{d \times m}$ has randomly drawn columns $\Vec{\omega}_j \distiid \Lambda$ for some probability distribution $\Lambda$ (usually Gaussian).  

\textit{Remark:} In our definition of the RFF, we also added a random \emph{dither} $\Vec{\xi}$, which has uniform entries $\xi_j \distiid \cl U([0,2\pi))$. Although not strictly necessary at this point, this dithering will be useful when we consider quantization. It is also mandatory if the imaginary exponential in (\ref{eq:qRFF}) is replaced by a sine or a cosine, as initially formulated in~\cite{Rahimi2008RFF}.

After this ``sketching phase'', the ``learning phase'' of CL extracts the desired machine learning model parameters $\bs \theta \in \Theta$ from the sketch. This is achieved by solving an optimization problem of the form
\begin{equation*}
\ts \min_{\Vec{\theta} \in \Theta} \cl C_{\Phi}(\Vec{\theta} ; \Vec{z}_{\Phi,\cl X} ).
\end{equation*}
Since the cost $\cl C_{\Phi}$ involves only the sketch of size $m \ll nd$ (\ie much smaller than the volume of $\ds$), this procedure is typically much more efficient from a computational point of view than the classical approach of learning directly from the entire dataset $\ds$, especially for large $n$.

Intuitively, given a map $\Phi$ (\eg $\Phi_{\RFF}$), the cost $\cl C_{\Phi}(\Vec{\theta} ; \Vec{z})$ captures the mismatch between the vector $\Vec{z}$ and another sketch, obtained using $\Phi$, associated with the candidate model $\Vec{\theta}$. Postponing the technical details for the moment, to solve k-means for instance (where one seeks a set of $K$ centroids $\Vec{\theta} = \{ \Vec{c}_k \}_{k=1}^K \subset \bb R^d$ that best cluster the data), this cost is given by $\cl C_{\Phi}(\Vec{\theta}; \Vec{z}) = \| \Vec{z} - \frac{1}{K} \sum_k \Phi(\Vec{c}_k) \|_2 $, \ie the Euclidean distance between the sketch of the data and the ``sketch of the centroids''~\cite{keriven2016compressive}.

Numerical experiments demonstrated the power of CL, sometimes solving k-means clustering~\cite{keriven2016compressive} and Gaussian mixture modeling (GMM)~\cite{keriven2016GMMestimation} with training time and memory consumption reduced by several orders of magnitude compared to classical approaches. Moreover, formal \emph{statistical learning guarantees} (bounds on the so-called excess risk) were derived~\cite{gribonval2017compressiveStatisticalLearning}: the risk is controlled whenever a form of Lower Restricted Isometry Property (LRIP) holds, which translates the compatibility between the sketch map $\Phi$ and the target learning task. Proving this LRIP is however quite technical; \eg see~\cite{gribonval2020statistical} for the LRIP between $\Phi_{\RFF}$ and k-means and GMM. 
This makes the design of a compressive learning schemes a bit rigid, as any deviation from the beaten path breaks the strong, but tediously constructed, theoretical guarantees. The ambition of this paper is to somehow relax this constraint on a specific aspect: allowing the designer to tweak the sketching phase (\eg to further improve its computational efficiency) without having to re-prove the LRIP-based guarantees from scratch.

More precisely, we study the scenario where the feature map for the sketching phase (now noted $\Psi$) is allowed to differ from the one used for the learning phase (noted $\Phi$), \ie we consider the \emph{asymmetric} CL (ACL) strategy modeled by $\min_{\Vec{\theta}} \cl C_{\Phi}(\Vec{\theta} ; \Vec{z}_{\Psi,\cl X} )$, with $\Psi \neq \Phi$. This scheme is interesting from a practical point of view because both phases, occurring in different contexts, may require very different properties from their respective feature maps.

\begin{figure}
	\centering
	\includegraphics[width=0.96\linewidth]{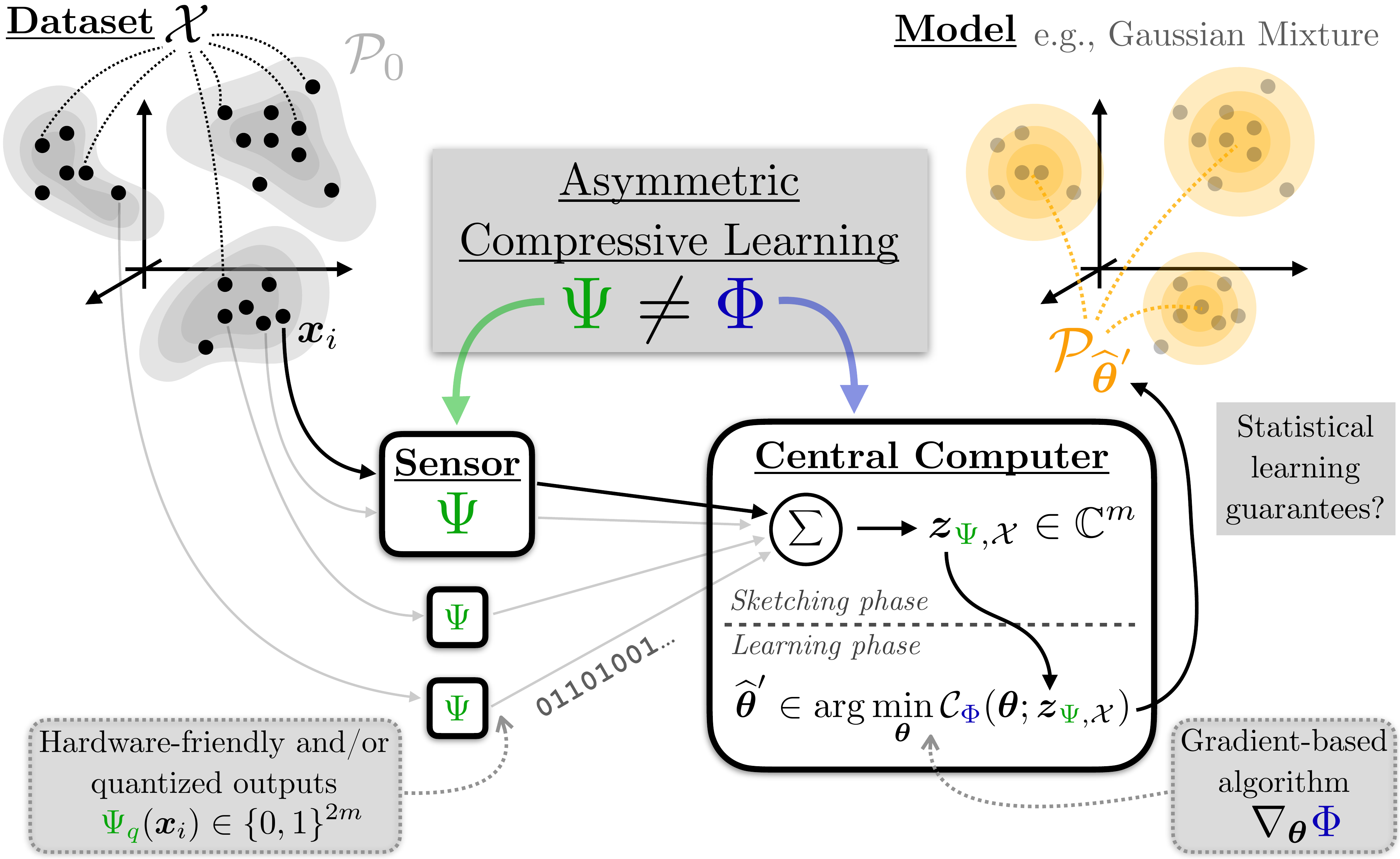}
	\caption{Our Asymmetric Compressive Learning (ACL) scheme: a dataset $\ds$ of $n$ examples $\bs x_i$ (sampled \textit{i.i.d.} from $\cl P_0$) is first compressed as a lightweight vector---the \emph{sketch}---by averaging data features $\Psi(\bs x_i)$. This operation can be performed in parallel by a sensor network, which benefits greatly from hardware-friendliness and quantization. A model $\widehat{\bs \theta}'$ is then learned from the sketch $\bs z_{\Psi,\ds}$ by solving a CL optimization procedure that uses a different, differentiable map $\Phi \neq \Psi$. Our goal is to prove statistical learning guarantees (\textit{w.r.t.} $\cl P_0$) for the model $\widehat{\bs \theta}'$.}
	\label{fig:intro}
\end{figure}

This fact is best explained by a concrete example, illustrated Fig.~\ref{fig:intro}: consider a sensor network, where each node collects a few data samples $\bs x_i$, and sends their contributions $\Psi(\bs x_i)$ to a centralized server which aggregates them to construct the sketch $\Vec{z}_{\Psi,\ds} = \frac{1}{n} \sum_{i=1}^n \Psi(\bs x_i)$. For efficient transmission of those numerous messages, \emph{quantization} of the contributions $\Psi(\bs x_i)$ is critical. Moreover, to ensure low power consumption of the sensor nodes, a compact \emph{hardware implementation} of $\Psi$ is highly desirable.

The learning phase however, being performed locally and in software, does not benefit as much from these aspects. Instead, the cost $\cl C_{\Phi}(\Vec{\theta} ; \Vec{z}_{\Psi,\cl X})$ must have \emph{meaningful minimizers} (\ie as ensured by statistical guarantees), and must be efficient to optimize; \eg gradient-based algorithms~\cite{keriven2016GMMestimation,keriven2016compressive} require \emph{differentiability} of $\Phi$. All these objectives do not necessarily align, and some are even incompatible (for example, differentiability cannot be reconciled with the discontinuity induced by quantization), hence the interest of allowing different sketching and learning maps $\Psi \neq \Phi$. 

In previous work~\cite{schellekens2018quantized}, we replaced the RFF sketch contributions $\Phi_{\RFF}(\bs x_i) \in \bb C^m$ by quantized ones $\Psi_q(\bs x_i) \in \frac{1}{\sqrt{m}} \{\pm 1 \pm \im \}^m$, obtained by taking the sign of usual RFF:
\begin{equation}
\label{eq:qRFF}
	\ts \Psi_q(\bs x) := \frac{1}{\sqrt{m}} \, \sign\left( \exp(\im (\bs\Omega^{\top} \bs x + \bs \xi )) \right),
\end{equation}
where the sign of a complex number $z \in \bb C$ is applied component-wise to its real and imaginary part, \ie $\sign(z) = \sign(\Re(z)) + \im\,\sign(\Im(z)) \in \{\pm 1 \pm \im \}$.
The embedding $\Psi_q$ is also known as (the complex extension of) the one-bit universal embedding studied in~\cite{boufounos2012universal,boufounos2017representation}. We showed numerically on the k-means task that the performance degraded only slightly, and provided an intuitive justification by showing $\cl C_{\Phi_{\RFF}}(\bs \theta; \bs z_{\Psi_q,\ds}) \simeq \cl C_{\Phi_{\RFF}}(\bs \theta; \bs z_{\Phi_{\RFF},\ds})$, for a fixed dataset $\ds$ and parameter vector $\bs \theta$. In fact, this approximation of the ``un-quantized'' costs is enabled by the dithering $\bs \xi$---a crucial ingredient of our scheme---which picks out (on average) the fundamental frequency of $q(t) = \sign(e^{\im t})$. However, this fact did not \emph{guarantee} anything regarding the excess risk achieved by ACL. This is remedied by the present paper; its main goal is to provide formal guarantees for the ACL scheme, by building upon the existing LRIP. This ensures that our result can be applied to other and future sketch constructions, provided these satisfy the LRIP.

\emph{Remark:} As will become clearer below, our results allow to replace, in the sensor implementing $\Psi$, the complex exponential $\exp(\im t)$ by a \emph{completely generic} periodic map $f(t)$ (of which $q(t) = \sign(\exp(\im t))$ defined above is a particular case). Remark that our results thus apply in particular to the case where, while the sensor can ensure the periodicity of the embedded map $f$, the precise shape of $f$ cannot be fully controlled, such that can’t really rely the sensor to accurately implement any \emph{specific} map---\eg due to imperfections and non-linearities.

\textbf{Contributions:}
After reviewing compressive learning background, we formalize the ACL problem in Sec.~\ref{sec:background}.

We first prove a general statistical guarantee for ACL in Sec.~\ref{sec:main-general}; to do so, we introduce a Limited Projected Distortion (LPD) property (capturing, roughly speaking, the similarity between $\Phi$ and $\Psi$), which we combine with the existing LRIP (which holds for $\Phi$). This result is ``general'' in the sense that it makes no assumption on the task to solve or the maps $\Phi$ and $\Psi$ (beyond the LRIP and LPD).

We then work towards concrete applications of this generic result in Sec.~\ref{sec:main-quantized}, dedicated to proving that the LPD actually holds for specific choices of $\Phi$ and $\Psi$. To achieve this, we first introduce---at the cost of an additional (but mild) assumption---a sufficient condition for the LPD, called ``signal-level LPD'' (sLPD). Next, we prove the sLPD on one particular combination of feature maps, where the ``learning phase feature map'' are random Fourier features $\Phi = \Phi_{\RFF}$, and the ``sketching phase feature map '' are \emph{random periodic features} $\Psi = \Psi_{f}$, \ie a modification of the RFF where the complex exponential is replaced by a generic periodic function $f$.
This allows to obtain at last, as a corollary of all the results above, formal statistical learning guarantees for the ACL scheme for the pair $(\Phi_{\RFF}, \Psi_q)$ considered in~\cite{schellekens2018quantized}, which ensures the theoretical soundness of that approach. To demonstrate the broader applicability of our results, we also apply them to the pair $(\Phi_{\RFF}, \Psi_{\smod})$, where $\Psi_{\smod}$ represents the---hardware-friendly---modulo measurement map studied (among others) in~\cite{shah2019signal}.

Moving on, we provide in Sec.~\ref{sec:experiments} further extensive empirical validations for the ACL strategy, confirming that, as our theory supports, it can be extended to other tasks (such as GMM) as well as other feature maps (such as $\Psi_{\smod}$ the modulo measurements map). We also highlight the practical advantage of our quantized compressive learning strategy on a large-scale audio classification task. Finally, we conclude in Sec.~\ref{sec:conclusion}.

\section{Background and problem statement}
\label{sec:background}

\renewcommand\arraystretch{1.2}

\begin{table}[]
  \centering
  \definecolor{tableShade}{gray}{0.95}
  \rowcolors{2}{tableShade}{white}
  \begin{tabularx}{.96\linewidth}{ll}
    \hiderowcolors 
    {\bf Concepts}&{\bf Description}\\
    \showrowcolors 
    \toprule
    $\setdistributions$&Probability measure set.\\
    $\cl P, \cl Q \in \setdistributions$&Arbitrary distributions in $\setdistributions$.\\
    $\disttrue \in \setdistributions$&Data distribution.\\
    $\cl P_{\Vec{\theta}} \in \setdistributions,\ \bs \theta \in \Theta$&Parametric distribution (see Table~\ref{tab:tasks}).\\
    $\modelset := \{ \cl P_{\Vec{\theta}} \, | \, \Vec{\theta} \in \Theta \} \subset \setdistributions$&Model set.\\
    $\measuredset \subset \setdistributions$&Empirical set (see Sec.~\ref{sec:main-general}).\\
    $\cl X = \{\bs x_i\}_{i=1}^n$&Dataset with $\bs x_i \sim_{\iid} \disttrue$.\\
    $\distemp := \frac{1}{n} \sum_{i=1}^n \delta_{\Vec{x}_i} \in \measuredset$&Empirical distribution of $\cl X$.
  \end{tabularx}
  \caption{Main concepts and distributions used in this work.}
  \label{tab:main-concepts}
\end{table}

We here review the background of compressive learning. For the sake of clarity, the main concepts and distributions introduced below are summarized in Table~\ref{tab:main-concepts}.

\textbf{Notations and definitions:}
We consider learning examples living in the ambient space $\bb R^d$.
The set of probability measures over $\bb R^d$ is noted $\setdistributions$, and $\delta_{\Vec{c}} \in \setdistributions$ is the Dirac delta measure located at $\Vec{c} \in \bb R^d$. To characterize the intrinsic dimension of a compact set $\Sigma \subset \bb R^d$, we use the Kolmogrov $\nu$-entropy~\cite{kolmogorov1961}: for any radius $\nu>0$ it is given by $\cl H_{\nu}(\Sigma) := \log \cl C_{\nu}(\Sigma) < \infty$, where $\cl C_{\nu}(\Sigma)$ is the covering number of $\Sigma$ by Euclidean balls of radius $\nu$, \ie 
\begin{equation}
\label{eq:covering_number}
	\ts \cl C_{\nu}(\Sigma) := \min \{ |\cl S| : \cl S \subset \Sigma \subset \cl S + \nu \bb B_2^d \},
\end{equation}
where the cardinality of $\cl S$ is written $|\cl S|$, the $d$-dimensional unit ball \textit{w.r.t.} the $\ell_p$-norm is $\bb B_p^d$, and the Minkowski sum of two sets $\cl A$ and $\cl B$ is $\cl A + \cl B = \{a + b: a \in \cl A, b \in \cl B\}$.

When dealing with \emph{generic periodic functions} $f : \bb R \rightarrow \bb C$, we assume without loss of generality that $f$ is normalized such that it is centered and with period given by $2\pi$; it can thus be decomposed as Fourier series $\ts f(t) = \sum_{k \in \bb Z} F_k e^{\im k t}$ where $F_k := \frac{1}{2\pi} \int_{0}^{2\pi} f(t) e^{-\im k t} \mathrm{d}t$ and $F_0 = 0$. For such functions, we define the \emph{mean Lipschitz smoothness property}, introduced in~\cite{schellekens2020breaking} to characterize the ``smoothness on average'' of (possibly discontinuous) functions:

\begin{definition}
	\label{def:meanSmoothness}
	A $2\pi$-periodic function $f : \bb R \rightarrow \bb C$ is \emph{mean Lipschitz smooth} with mean Lipschitz constant $L^{\mu}_f$ if for all radii $\delta \in (0,\pi]$ the maximum deviation of $f$ in the interval $[-\delta,\delta]$ is, on average, bounded by $L^{\mu}_f \delta$, \ie
	\begin{equation}
		\label{eq:meanSmooth}
		\tfrac{1}{2\pi} \ts \int_{0}^{2\pi} \sup_{r \in [-\delta, \delta]} \{ |f(t + r) - f(t)| \} \: \mathrm{d}t \leq L^{\mu}_f \cdot \delta.
	\end{equation}
\end{definition}
The advantage of this particular ``smoothness'' criterion is that is allows to handle \emph{discontinuous functions}; \eg the maps depicted Fig.~\ref{fig:periodic_maps} are mean smooth (see App.~\ref{sec: comp-mod-q-constant}).

By abuse of notation, evaluating a scalar function $f : \bb R \rightarrow \bb C$ on a vector $\Vec{u} \in \bb R^m$ means applying this function componentwise, \ie $f(\Vec{u}) \in \bb C^m$ with $(f(\Vec{u}))_j = f(u_j)$. Similarly, ``inequalities'' between vectors $\Vec{u}, \Vec{v} \in \bb R^m$ are to be interpreted component-wise, \eg $\bs u \leq \bs v$ means $u_j \leq v_j$ for all $1 \leq j \leq m$.

\subsection{Statistical Learning}
In the \emph{statistical learning} (SL) framework~\cite{vapnik1999overview,shalev2014understanding}, one assumes that the signals of interest are generated by a \emph{data distribution} $\disttrue \in \setdistributions$. 
The goal is then to fit some machine learning model, parametrized by a vector $\bs \theta  \in \Theta$, to that distribution. More precisely,
one seeks the model parameters $\soltruerisk$ that minimize the \emph{risk} objective $\risk(\Vec{\theta};\disttrue) := \expec{\Vec{x} \sim \disttrue} \loss(\Vec{x},\Vec{\theta})$, \ie the expectation of a loss $\loss : \bb R^d \times \Theta \rightarrow \bb R$ with respect to the data distribution:
\begin{equation}
\label{eq:true_risk_minimization}
 \soltruerisk \in \arg\min_{\Vec{\theta} \in \Theta} \risk(\Vec{\theta};\disttrue) = \arg\min_{\Vec{\theta}\in \Theta} \: \ts \expec{\Vec{x} \sim \disttrue} \loss(\Vec{x},\Vec{\theta}).
\end{equation}

In practice, the true data distribution $\disttrue$ is unknown, but a dataset $\ds = \{\Vec{x}_i\}_{i=1}^n$ of $n$ samples $\Vec{x}_i \distiid \disttrue$ is available. The ``ideal'' risk minimization~\eqref{eq:true_risk_minimization} is thus replaced by \emph{empirical risk minimization} (ERM), which uses the empirical distribution $\distemp := \frac{1}{n} \sum_{i=1}^n \delta_{\Vec{x}_i} \in \setdistributions$ instead of the true data distribution:
\begin{equation}
\label{eq:ERM}
\solerm \in \arg\min_{\Vec{\theta} \in \Theta} \risk(\Vec{\theta};\distemp) = \arg\min_{\Vec{\theta} \in \Theta} \ts \sum_{\Vec{x}_i \in \ds} \loss(\Vec{x}_i,\Vec{\theta}).
\end{equation}
A regularization term can also be added to \eqref{eq:ERM}, \eg to avoid overfitting $\cl X$.

Many common machine learning task can be cast into the SL framework. In classification for instance, $\ell(\bs x; \bs \theta) \in \{0,1\}$ is the 0-1 loss function, equal to $1$ (resp. $0$) whenever the decision function associated with $\bs \theta$ classifies $\bs x$ correctly (resp. incorrectly). In this work we focus on two unsupervised tasks: k-means and Gaussian mixture modeling.

As summarized in Table~\ref{tab:tasks}, \textit{k-means clustering} seeks $K$ centroids $\bs c_k \in \bb R^d$ which minimize the sum of squared errors (SSE) over the dataset (the ``error'' is the distance between each sample $\bs x_i$ and the centroid closest to it):
\begin{equation}
\label{eq:SSE}
	\SSE(\bs \theta; \ds) := \ts \sum_{i = 1}^n \min_{1 \leq k \leq K} \|  \bs x_i - \bs c_k \|_2^2
\end{equation}

On the other hand, \textit{Gaussian mixture modeling} (GMM) seeks a weighted mixture of $K$ Gaussians $\cl N(\bs \mu_k, \bs\Gamma_k)$ (\ie weights $w_k \geq 0$ that sum to one, centers $\bs \mu_k\in \Rbb^d$, and positive definite covariance matrices $\bs \Gamma_k\in \Rbb^{d\times d}$) that \emph{maximizes} the log-likelihood (LL) of the dataset $\ds$: 
\begin{equation}
	\label{eq:loglikelihood}
	\LL(\bs \theta; \ds) := \ts \sum_{i = 1}^n \log \big( \sum_{k=1}^K w_k p_{\cl N}(\bs x_i; \bs \mu_k, \bs \Gamma_k) \big),
\end{equation}
where $p_{\cl N}(\bs x; \bs \mu, \bs \Gamma)$ is the probability density function of the Gaussian distribution $\cl N(\bs \mu, \bs \Gamma)$ evaluated at $\bs x$.

The central goal of SL is to control the \emph{excess risk} $\cl R(\solerm;\disttrue) - \cl R(\soltruerisk;\disttrue)$ (also known as generalization error for prediction tasks), in the form of \emph{statistical guarantees}: for some $\delta \in (0,1)$ and $\eta > 0$, the ERM solution $\solerm$ satisfies 
\begin{equation}
\label{eq:statguarantee}
\bb P[ \cl R(\solerm;\disttrue) - \cl R(\soltruerisk;\disttrue) \leq \eta] \geq 1 - \delta. 
\end{equation}
In words, this guarantee ensures that, with probability larger than $1-\delta$ over the sampling of $\ds$, the estimate of the ERM is not worse than the optimal solution $\soltruerisk$ (on the true data distribution $\disttrue$) by a margin smaller than $\eta$; this is also called Probably Approximately Correct (PAC) learning~\cite{shalev2014understanding}.

\subsection{Compressive Learning}
The computational resources (such as memory and time) required to solve~\eqref{eq:ERM} most often increase with $n$ (consider for example that merely evaluating of the ERM cost at any one solution $\bs\theta$ already requires a full pass over the $n$ learning examples). Compressive learning (CL)~\cite{gribonval2017compressiveStatisticalLearning} avoids this issue, since the $m$-dimensional sketch is first computed in a single distributable and parallelizable pass; subsequent learning then scales only with its much smaller size $m \ll nd$ (in fact, $m$ does not dependent on $n$ at all).

More precisely, CL theory~\cite{gribonval2017compressiveStatisticalLearning} actually introduces a general \emph{sketch operator} $\sketchop[\Phi]$, which acts on the space probability distributions $\setdistributions$. This operator ``compresses'' any input distribution $\cl P \in \setdistributions$ by computing $m$ of its \emph{generalized moments}, as defined by the associated feature map $\Phi$.
\begin{definition}[Sketch]
	Given a feature map $\Phi : \bb R^d \rightarrow \bb C^m$, the associated \emph{sketch operator} $\sketchop[\Phi] : \setdistributions \rightarrow \bb C^m$ is 
	\begin{equation}
	\label{def:sketch_op}
	\ts \sketchop[\Phi](\cl P) \eqdef \expec{\Vec{x} \sim \cl P} \Phi( \Vec{x} ) = \Integr{}{}{\Phi(\Vec{x})}{\cl P(\Vec{x})} \in \bb C^m.
	\end{equation}
	In particular, the \emph{sketch of a dataset} $\ds = \{\Vec{x}_i\}_{i=1}^n$, noted $\Vec{z}_{\Phi,\ds}$, is actually the sketch of its empirical distribution,
	\begin{equation}
	\label{def:sketch_ds}
	\ts \Vec{z}_{\Phi,\ds} \eqdef \sketchop[\Phi](\distemp) = \frac{1}{n} \sum_{i = 1}^{n} \Phi( \Vec{x}_i ) \in \bb C^m.
	\end{equation}
\end{definition}

We now provide an informal outline (not fully rigorous but sufficient for our purposes; see~\cite{gribonval2017compressiveStatisticalLearning} for details) of how to ``learn'' (find good parameters $\bs \theta$) from the sketch. One first associates to each parameter vector $\bs \theta \in \Theta$ a distribution $\cl P_{\Vec{\theta}} \in \setdistributions$ (this map $\bs \theta \mapsto \cl P_{\bs \theta}$ is not necessarily injective~\cite{sheehan2019compressive}), which respects a \emph{risk consistency} property:
\begin{equation}
\label{eq:risk-consistency}
	\risk(\bs \theta; \cl P_{\bs \theta}) \leq \risk(\bs \theta'; \cl P_{\bs \theta}),\quad \forall \bs \theta' \in \Theta.
\end{equation}

\renewcommand\arraystretch{1.75}
\begin{table}[]
	\centering
	\begin{smaller}
		\begin{tabular}{l|c|c|}
			\cline{2-3}
			& \textbf{k-means clustering}~\cite{keriven2016compressive} & \textbf{Gaussian Mixture Modeling}~\cite{keriven2016GMMestimation} \\ \hline
			\multicolumn{1}{|l|}{$\bs \theta$} & centroids $\{\bs c_k \}_{k=1}^K$ & params. $\{w_k, \bs \mu_k, \bs \Gamma_k \}_{k=1}^K$ \\ \hline
			\multicolumn{1}{|l|}{\multirow{2}{*}{$\Theta$}}
			& \multirow{2}{*}{$\bs c_k \in \bb R^d$} & $w_k \geq 0, \sum_k w_k = 1, \bs \mu_k \in \bb R^d$\\ \multicolumn{1}{|l|}{} && $ \bs \Gamma_k \in \bb R^{d \times d}, \, \bs \Gamma_k^{\top} = \bs \Gamma_k \succeq 0 $ \\ \hline
			\multicolumn{1}{|l|}{$\loss(\bs x, \bs \theta)$}  & $\min_k \| \bs x - \bs c_k \|_2^2$ & $- \log \sum_k w_k p_{\cl N}(\bs x; \bs \mu_k, \bs \Gamma_k)$ \\ \hline
			\multicolumn{1}{|l|}{$\cl P_{\bs \theta}$}  & $\sum_{k=1}^K \frac{1}{K} \delta_{\Vec{c}_k} $   & $\sum_k w_k \cl N(\bs \mu_k, \bs \Gamma_k)$  \\ \hline
			\multicolumn{1}{|l|}{$\cl C_{\Phi}(\Vec{\theta} ; \Vec{z})$}  & $ \| \Vec{z} - \frac{1}{K} \sum_k \Phi(\Vec{c}_k) \|_2 $   & $ \| \Vec{z} - \sum_k w_k \sketchop[\Phi](\cl N(\bs \mu_k, \bs \Gamma_k)) \|_2 $  \\ \hline
		\end{tabular}
	\end{smaller}
	\caption{Description of two SL tasks and their equivalent in the CL framework.
	K-means seeks the $K$ centroids $c_k$ that minimize the sum of squared errors $\SSE(\bs \theta; \ds)$; in CL the parameters are mapped to a sum of $K$ Dirac deltas. GMM seeks a weighted mixture of $K$ Gaussians $\cl N(\bs \mu_k, \bs \Sigma_k)$ that maximize the log-likelihood of the data $\LL(\bs \theta; \ds)$; in this case CL simply maps the parameters to the GMM distribution itself.}
	\label{tab:tasks}
\end{table}

Table~\ref{tab:tasks} gives examples of this map for k-means\footnote{To emphasize the connection between the SL and CL formulations of k-means, we assign in Table~\ref{tab:tasks} equal weights $\frac{1}{K}$ to all the Dirac deltas of the centroids $\delta_{\Vec{c}_k}$; however the complete formulation of CL k-means actually considers those weights as free parameters to be optimized~\cite{keriven2016compressive}.} and GMM. When the parameter vector $\bs \theta$ varies, the resulting distributions constitute a \emph{model set} $\modelset := \{ \cl P_{\Vec{\theta}} \, | \, \Vec{\theta} \in \Theta \} \subset \setdistributions$.
Learning then amounts to finding the parametrized distribution $\cl P_{\bs \theta}$ from $\modelset$ whose sketch---with respect to $\Phi$---best fits the dataset sketch $\Vec{z}_{\Phi,\ds}$, as defined by the cost $\cl C_{\Phi}$:
\begin{equation}
\label{eq:sketchmatching}
	\solsketch \in \arg\min_{\Vec{\theta} \in \Theta} \; \cl C_{\Phi}(\Vec{\theta} ; \Vec{z}_{\Phi,\cl X} ) := \| \Vec{z}_{\Phi,\cl X} - \sketchop[\Phi](\cl P_{\Vec{\theta}}) \|_2.
\end{equation}

Guarantees of the form~\eqref{eq:statguarantee} can be proven for this \emph{sketch matching principle}. The idea is to show~\eqref{eq:sketchmatching} is a surrogate approximating~\eqref{eq:ERM}. Intuitively, it is possible to solve a task from the sketch if it somehow ``encodes'' that tasks risk objective. To assess how well the risk is encoded,~\cite{gribonval2017compressiveStatisticalLearning} defines a seminorm $\|\cdot\|_{\cl R}$ to measure the difference between two distributions $\cl P, \cl Q$ with respect to the task-specific risk $\cl R$,
\begin{equation}
\label{eq:excessRiskMetric}
\|\cl P - \cl Q\|_{\risk} := \ts \sup_{\bs\theta \in \Theta} \, | \risk(\bs\theta;\cl P) - \risk(\bs\theta;\cl Q) |.
\end{equation}
Equipped with this metric, we say that the sketch operator $\sketchop[\Phi]$ ``encodes'' the risk $\risk$ if the sketch distance $\|\sketchop[\Phi](\cl P)-\sketchop[\Phi](\cl Q)\|_2$ bounds $\|\cl P - \cl Q\|_{\risk}$ for all distributions in $\modelset$; the Lower Restricted Isometry Property (LRIP) formalizes this notion (this specific LRIP generalizes its well-known equivalent from compressive sensing literature~\cite{candes2005decoding,foucart2017mathematical}).

\begin{definition}[LRIP]
	\label{def:LRIP}
	The sketch operator $\sketchop[\Phi]$ has the LRIP with constant $\lripc$ on the model set $\modelset$, noted $\LRIP(\lripc;\modelset)$, if
	\begin{equation}
	\label{eq:LRIP}
	\forall \cl P, \cl Q \in \modelset, \quad \| \cl P - \cl Q \|_{\cl R} \leq \lripc \| \sketchop[\Phi] (\cl P) - \sketchop[\Phi] (\cl Q)  \|_2.
	\end{equation} 
\end{definition}

One key theorem of CL\footnote{For the sake of presentation, Thm.~\ref{prop:risk_gribonval} is taken from~\cite[Sec.\,2.4]{gribonval2017compressiveStatisticalLearning} which presents a simplified but sub-optimal version of the ``true'' CL guarantees, established in~\cite[Sec.\,2.5]{gribonval2017compressiveStatisticalLearning}. However, the extension we prove in this paper can be carried over seamlessly to the main CL guarantees, as the improvements from~\cite[Sec.\,2.5]{gribonval2017compressiveStatisticalLearning} are independent of our developments.}, proved in~\cite{gribonval2017compressiveStatisticalLearning}, is that the LRIP implies statistical guarantees for the sketch matching~\eqref{eq:sketchmatching}.

\begin{theorem}[LRIP implies excess risk control]
\label{prop:risk_gribonval}
	Assume that $\sketchop[\Phi]$ has the $\LRIP(\lripc;\modelset)$. The excess risk of the solution $\solsketch$ to~\eqref{eq:sketchmatching} satisfies $\cl R(\solsketch;\disttrue) - \cl R(\soltruerisk;\disttrue) \leq \eta$, where
	\begin{equation}
	\label{eq:excessrisk_gribonval}
		\eta = 2 D(\cl P_0, \modelset) + 4 \lripc \| \sketchop[\Phi](\disttrue) - \sketchop[\Phi](\distemp) \|_2,
	\end{equation}	
	with $D(\cl P,\modelset)$ a ``distance'' from $\cl P$ to the model set $\modelset$,
	$$D(\cl P,\modelset) := \inf_{\cl Q \in \modelset} \{ \| \cl P - \cl Q \|_{\cl R} + 2 \lripc \|\sketchop[\Phi](\cl P) - \sketchop[\Phi](\cl Q) \|_2 \}.$$
\end{theorem}

The first term in~\eqref{eq:excessrisk_gribonval} is a modeling bias term, the second one captures a sampling error, which decreases with $n$. 

Theorem~\ref{prop:risk_gribonval} guarantees that the excess risk is under control (bounded by $\eta$) provided that the related LRIP holds; it remains thus to prove the latter.  This endeavor is highly specific to the considered model $\modelset$ (\ie the learning task) and feature map $\Phi$. One usually proves that the LRIP holds \emph{with high probability} $1-\delta$ on the random draw of $\Phi$, where the failure probability $\delta$ depends on the desired LRIP constant $\lripc$, the complexity of the model set $\modelset$, and the number of ``measurements'' $m$. These proofs are rather technical, see~\cite{gribonval2020statistical} for the case of compressive k-means and Gaussian mixture modeling from RFF sketches.

\emph{Remark:} While we focus on theoretical guarantees for to the sketch matching program~\eqref{eq:sketchmatching}, it is worth noting that this optimization problem is usually nonconvex. In practice, heuristics, such as compressive learning orthogonal matching pursuit (\texttt{CLOMP})~\cite{keriven2016GMMestimation,keriven2016compressive}, thus \emph{approximately} solve~\eqref{eq:sketchmatching}. Although they showed empirical success, one should keep in mind that those heuristics do not necessarily find the global solution $\solsketch$, and that there might still be a performance gap between experimental results and the theoretical statistical learning guarantees (which apply to the global solution $\solsketch$). This paper mainly focuses on theoretical guarantees, except for the numerical validation (Sec.~\ref{sec:experiments}).

\subsection{Asymmetric CL with distorted or quantized sketches}

To summarize, the typical CL scenario consists of two phases: the sketching phase where the dataset $\ds$ is harshly compressed to a sketch vector $\bs z_{\Phi,\ds}$ (by averaging some feature map $\Phi$ over the data samples), followed by the learning phase, where the desired machine learning model (parametrized by $\bs \theta$) is then extracted by solving $\solsketch \in \arg\min_{\Vec{\theta} \in \Theta} \cl C_{\Phi}(\Vec{\theta} ; \Vec{z}_{\Phi,\cl X} )$, \ie the sketch matching principle described in~\eqref{eq:sketchmatching}. The success of this scheme can thus be guaranteed by establishing the LRIP of the sketch operator $\sketchop[\Phi]$ associated to the ``reference'' feature map $\Phi : \bb R^d \rightarrow \bb C^m$ (\eg the random Fourier features, $\Phi_{\RFF}$), which must hold over the relevant task's model set~$\modelset$.

In this work, we question the possibility to extend this scheme by allowing the sketching phase to use a different---or ``distorted''---feature map $\Psi \neq \Phi$ (\eg the binarized RFF, $\Psi_q$ from~\eqref{eq:qRFF}). This amounts to studying this \emph{asymmetric compressive learning} scenario (ACL for short): given a \emph{reference feature map} $\Phi$ and a different \emph{distorted feature map} $\Psi$, and having observed the ``distorted sketch'' $\bs z_{\Psi,\cl X} = \sketchop[\Psi](\distemp) = \frac{1}{n} \sum_{i=1}^n \Psi(\bs x_i) $, we select the parameters $\solquant$ that solve the ``asymmetric sketch matching'' problem, \ie
\begin{equation}
	\label{eq:asymsketchmatching}
	\solquant \in \arg\min_{\Vec{\theta} \in \Theta} \cl C_{\Phi}(\Vec{\theta} ; \Vec{z}_{\Psi,\cl X} ) = \| \Vec{z}_{\Psi,\cl X} - \sketchop[\Phi](\cl P_{\Vec{\theta}}) \|_2.
\end{equation}
To be perfectly clear, the ``asymmetry'' here refers to the fact that \emph{only} the sketching map is distorted, since we still use the reference map $\Phi$ to learn from $\Vec{z}_{\Psi,\cl X}$, as precised by the subscript in the cost $\cl C_{\Phi}$ (the only difference with the symmetric sketch matching from~\eqref{eq:sketchmatching} is the sketch map).

This (perhaps surprising) strategy is inspired by a well-known equivalent in classical signal estimation: provided a signal follows a low-complexity model (such as a sparse or a low-rank description), one can treat its nonlinearly distorted measurements as noisy linear observations, \eg ignoring quantization, with provable reconstruction guarantees if the involved nonlinearity respects a few mild conditions~\cite{plan2016generalizedLasso}. Our work adopts a similar approach to get risk control guarantees for~\eqref{eq:asymsketchmatching}, of the same form as Thm.~\ref{prop:risk_gribonval}. This enables to quantify when (and under which conditions on $\Psi$) the asymmetric sketch matching scheme~\eqref{eq:asymsketchmatching} succeeds.

\textbf{RPF sketches:}
When the reference map $\Phi$ is the random Fourier features map $\Phi_{\RFF}$, our analysis allows us in particular to replace the complex exponential in~\eqref{eq:RFF} by any (properly normalized) $2\pi$-periodic function $f : \bb R \rightarrow \bb C$, \ie to consider as distorted feature map $\Psi$ the \emph{random periodic features} (RPF)~\cite{boufounos2017representation,schellekens2020breaking}, defined as
\begin{equation}
	\label{eq:RPF}
	\ts \Psi_f(\bs x) := \frac{1}{\sqrt{m}} \, f \left( \bs\Omega^{\top} \bs x + \bs \xi  \right).
\end{equation}
This choice is motivated by the following observation: the asymmetric cost approaches---in expectation over the uniform dither $\bs \xi \sim \cl U^m([0,2\pi))$---the symmetric cost~\cite{schellekens2018quantized}:
\begin{equation}
	\label{eq:qckm_motivation}
	\ts \expec{\bs \xi} \left[ \cl C_{\Phi_{\RFF}}(\Vec{\theta} ; \Vec{z}_{\Psi_f,\cl X} ) - \cl C_{\Phi_{\RFF}}(\Vec{\theta} ; \Vec{z}_{\Phi_{\RFF},\cl X} ) \right] = c_f, 
\end{equation}
with $c_f$ a constant shift depending only on $f$, that does not impact the optimization procedure.

Besides usual RFF (a special case of RPF), we consider two specific instances of RPF sketches, as illustrated Fig.~\ref{fig:periodic_maps}.

As mentioned in Sec.~\ref{sec:introduction}, a first interesting case is the \emph{quantized RFF} $\Psi_{q}$, defined by~\eqref{eq:qRFF}, for which 
\begin{equation}
  \label{eq:unfi-quant-def}
f(t) = q(t) := \sign(e^{\im t}) \quad \in \{\pm 1 \pm \im\},
\end{equation}
with $\sign$ acting independently on the real and imaginary component---\ie the real and imaginary components of $q$ are phase shifted ``square waves''.
This feature map is (the complex extension\footnote{To simplify the comparison we focus on \emph{complex-valued} maps ($\Psi_{q}$, $\Psi_{\smod}$), but their \emph{real-valued} counterpart ($\Re\Psi_{q}$, $\Re\Psi_{\smod}$), closer to their initial formulations in the literature, could very well also be considered.} of) the so-called ``one-bit universal quantization'' introduced in~\cite{boufounos2012universal}, and further studied in~\cite{boufounos2017representation}.

Sketching with $\Psi_{q}$ presents two advantages. First, it produces \emph{quantized} sketch contributions\footnote{The sketch itself (\ie after averaging) is not necessarily quantized.} $\Psi_q(\bs x_i) \in \{ \frac{\pm 1 \pm \im}{\sqrt{m}} \}^m$, which heavily reduces the cost of their (potential) transmission or storage (\eg $\Psi_q(\bs x_i)$ can be trivially encoded using only $2m$ bits). Second, the ``universal quantization'' operation $q(\cdot)$, which can be interpreted as the Least Significant Bit of a plain uniform scalar quantizer, is moreover amenable to \emph{plausible hardware implementations}, \eg through the use of voltage controlled oscillators~\cite{yoon2008time}. Moreover, even if $q$ is implemented in software, it is still cheaper to evaluate than $\exp(\im\cdot)$---a significant improvement since the complex exponential is the computational bottleneck in fast implementations of $\Phi_{\RFF}$~\cite{chatalic2018fastSketch}.

\begin{figure}
	\centering
	\includegraphics[width=\linewidth]{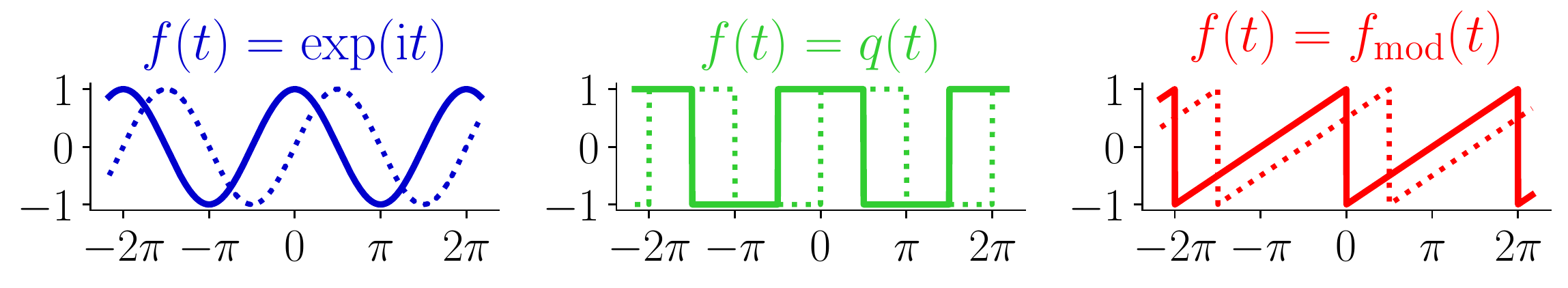}
	\caption{The $2\pi$-periodic functions $f : \bb R \rightarrow \bb C$ of particular interest ($\Re f$ in plain, $\Im f$ in dashed), related to RPF $\Psi_{f}(\bs x) = \frac{1}{\sqrt{m}} f(\bs\Omega^{\top}\bs x + \bs \xi)$. Left, the complex exponential $\exp(\im \cdot)$, related to the usual random Fourier features $\Phi_{\RFF}$; middle, the one-bit universal quantization function $q$, related to quantized RFF $\Psi_{q}$; and right, the (complex and normalized) modulo function $\fmod$, related to the modulo features $\Psi_{\smod}$.}
	\label{fig:periodic_maps}
\end{figure}

Another particular case of the RPF with promising applications is the (complex) \emph{modulo RFF} $\Psi_{\smod}$ defined as
\begin{equation}
  \label{eq:mod-f-def}
f(t) = \fmod(t) := \ts \tmod_{2\pi}(t) + \im \cdot \tmod_{2\pi}(t - \frac{\pi}{2})\ \in \bb C,
\end{equation}
where $\tmod_{T}(t) := 2\big(\frac{t}{T} - \lfloor \frac{t}{T} \rfloor\big)-1$ is the ``normalized'' modulo $T$ operation---\ie the real and imaginary components of $\fmod()$ are phase shifted ``sawtooth waves'' (Fig.~\ref{fig:periodic_maps}).
These modulo features take their roots from the recent theory of \emph{modulo sampling} of signals~\cite{bhandari2017unlimited}, and---much closer to our definition of $\Psi_{\smod}$---its extension to compressive modulo measurements of structured signals~\cite{shah2019signal}.

As explained in~\cite{bhandari2017unlimited,shah2019signal}, the advantage of this scheme is the existence of dedicated modulo sensors (\eg using self-reset analog-to-digital converters~\cite{rhee2003wide}) which again paves the way for \emph{efficient hardware computation} of the sketch contributions $\Psi_{\smod}(\bs x_i)$.

Finally, we recall that as remarked above, the RPF are also of interest if one seeks to implement $f$ in hardware sensors but can only ensure its periodicity, without accurate control of its precise shape $f$ due to imperfections.

\textbf{Previous work:}
In~\cite{schellekens2018quantized}, we demonstrated empirically, on the k-means problem, that replacing $\Phi_{\RFF}$ by $\Psi_q$ only induced a moderate decrease of the learning performances. However, the nice theoretical guarantees of CL (\ie Thm~\ref{prop:risk_gribonval}) do not hold in this setting anymore. While~\eqref{eq:qckm_motivation} is a promising start to gain intuition on why ACL could work, it does not \emph{guarantee} anything (\eg for finite values of $m$).

\section{A generic guarantee for ACL}
\label{sec:main-general}

Our goal is thus to prove guarantees in the spirit of Thm.~\ref{prop:risk_gribonval} for the ACL scheme. To derive unifying guarantees for each combination of model set $\modelset$ (task), reference map $\Phi$ and distortion $\Psi$, we build upon the existing LRIP (characterizing the compatibility between $\modelset$ and $\Phi$), which we combine with another property, the \emph{Limited Projected Distortion} (LPD) property (extending its definition from~\cite{xu2020quantized}). The LPD better characterizes the closeness between the distorted and reference sketches, $\Vec{z}_{\Psi,\cl X}$ and $\Vec{z}_{\Phi,\cl X}$, than the too restrictive Euclidean distance $\|\Vec{z}_{\Psi,\cl X}  - \Vec{z}_{\Phi,\cl X} \|_2$: for example, with the quantized RFF sketch~\cite{schellekens2018quantized}, this distance does not vanish as $m$ grows to infinity. 

The LPD relies on an additional assumption characterizing over which datasets $\ds$ it should hold.
\begin{assumption}
\label{as:empirical_set}
	There exists a set $\measuredset \subseteq \setdistributions$, coined \emph{empirical set}, such that, for any considered dataset $\ds$, the empirical distribution $\distemp$ belongs to $\measuredset$.
\end{assumption}
Note that Assumption~\ref{as:empirical_set} is quite permissive. We allow in particular to pick $\measuredset = \setdistributions$, which means it holds trivially. Moreover, even when $\measuredset \subsetneq \setdistributions$, it is quite mild; it is for example fulfilled under the natural assumption that the data samples $\bs x_i$ belong to a bounded domain $\Sigma$ (more on this later). That being said, we can now formalize the LPD.

\begin{definition}[LPD]
Given the empirical set $\measuredset \subseteq \setdistributions$, an error $\epsilon>0$, a reference sketch operator $\sketchop[\Phi]$, and a (task-dependent) model set $\modelset \subset \setdistributions$, we say that the distorted sketch operator $\sketchop[\Psi]$ satisfies the LPD of error $\epsilon$ on $\measuredset$ with respect to $\sketchop[\Phi]$ and $\modelset$, or shortly $\LPD(\epsilon;\measuredset,\modelset, \sketchop[\Phi])$, if
	\begin{equation}
	\label{eq:LPD}
	\ts \forall \cl P \in \measuredset, \cl Q \in \modelset, \quad | \langle \sketchop[\Psi](\cl P) - \sketchop[\Phi](\cl P) , \sketchop[\Phi](\cl Q) \rangle | \leq \epsilon.
	\end{equation}
\end{definition}

In words, the LPD thus ensures that, for any considered dataset $\ds$, the difference between its distorted sketch $\bs z_{\Psi,\ds}$ and its reference sketch $\bs z_{\Phi,\ds}$ is sufficiently small when projected on any possible ``reference sketch'' of the model set, \ie projected on any $\sketchop[\Phi](\cl P_{\bs \theta})$ for all parameters $\bs \theta \in \Theta$. As made clear below, $\{\sketchop[\Phi](\cl P_{\bs \theta}): \cl P_{\bs \theta} \in \modelset\}$ contains actually the ``directions'' that matter for solving~\eqref{eq:asymsketchmatching}.

Our first main result (Prop.~\ref{prop:risk_thiswork}) states that if the reference sketching operator $\sketchop[\Phi]$ satisfies the LRIP, and the distorted sketching operator $\sketchop[\Psi]$ the LPD, then the excess risk of the ACL solution~\eqref{eq:asymsketchmatching} can be controlled. To prove this, we adapt the proof of Thm~\ref{prop:risk_gribonval} found in~\cite{gribonval2017compressiveStatisticalLearning,keriven2018instanceOptimal}, and we leverage the LPD to show that distorting the sketch does not modify the cost function too much---as first shown in Lemma~\ref{lem:LRIPandLPD}.

\begin{lemma}
	\label{lem:LRIPandLPD}
	Assume $\sketchop[\Psi]$ satisfies the $\LPD(\epsilon;\measuredset,\modelset, \sketchop[\Phi])$. For any $\distemp \in \measuredset$, the asymmetric sketch matching solution $\solquant$ to~\eqref{eq:asymsketchmatching} is sub-optimal with respect to the symmetric matching solution $\solsketch$ to~\eqref{eq:sketchmatching} by at most 
	\begin{equation}
		\cl C_{\Phi}(\solquant ; \Vec{z}_{\Phi,\cl X} ) - \cl C_{\Phi}(\solsketch; \Vec{z}_{\Phi,\cl X} ) \leq 2\sqrt{\epsilon}.
	\end{equation}
	
\end{lemma}
\begin{proof}
	For conciseness, we drop the $\ds$ subscript and denote $\bs a := \sketchop[\Phi](\cl P_{\wh{\bs \theta}})$, $\bs a' := \sketchop[\Phi](\cl P_{\wh{\bs \theta'}})$. The LPD implies both
	\begin{align*}
          \ts \|\Vec{z}_{\Psi} - \bs a\|_2^2 - \|\Vec{z}_{\Phi} - \bs a\|_2^2&\ts \leq 2 \epsilon + \|\Vec{z}_{\Psi}\|_2^2 - \|\Vec{z}_{\Phi}\|_2^2,\\
          \ts \|\Vec{z}_{\Phi} - \bs a'\|_2^2 - \|\Vec{z}_{\Psi} - \bs a'\|_2^2&\ts \leq 2 \epsilon + \|\Vec{z}_{\Phi}\|_2^2 - \|\Vec{z}_{\Psi}\|_2^2.
	\end{align*}
	By optimality of~\eqref{eq:asymsketchmatching}, $\|\Vec{z}_{\Psi} - \bs a'\|_2^2 - \|\Vec{z}_{\Psi} - \bs a\|_2^2 \leq 0$, and adding the three inequalities together gives $\|\Vec{z}_{\Phi} - \bs a'\|_2^2 \leq \|\Vec{z}_{\Phi} - \bs a\|_2^2 + 4\epsilon$; a square root completes the proof.
\end{proof}

\begin{proposition}[Asymmetric sketch matching risk control]
\label{prop:risk_thiswork}
	Assume $\sketchop[\Phi]$ satisfies the $\LRIP(\lripc;\modelset)$ and $\sketchop[\Psi]$ the $\LPD(\epsilon;\measuredset,\modelset, \sketchop[\Phi])$. The solution $\solquant$ to the asymmetric problem~\eqref{eq:asymsketchmatching} satisfies $\cl R(\solquant;\disttrue) - \cl R(\soltruerisk;\disttrue) \leq \eta'$, where
	\begin{equation}
	\label{eq:excessrisk_asym}
	\eta' = 2 D(\cl P_0, \modelset) + 4 \lripc \| \sketchop[\Phi](\disttrue) - \sketchop[\Phi](\distemp) \|_2 + 4 \lripc \sqrt{\epsilon},
	\end{equation}	
	with $D(\cl P,\modelset)$ as in Thm.~\ref{prop:risk_gribonval}.
\end{proposition}
\begin{proof}
We use the same notations as in the previous proof. For some arbitrary $\cl Q \in \modelset$, since $\|\cdot\|_{\risk}$ is a seminorm~\cite{gribonval2017compressiveStatisticalLearning},
\begin{equation*}
\| \cl P_{\solquant} - \cl P_0 \|_{\cl R} \leq  \| \cl P_{\solquant} - \cl Q \|_{\cl R} + \| \cl Q - \cl P_0 \|_{\cl R}.
\end{equation*}
Since $\cl P_{\solquant}, \cl Q \in \modelset$, the LRIP and the triangle inequality give
\begin{equation*}
\begin{split}
	\| \cl P_{\solquant} - \cl Q \|_{\cl R} &\leq \lripc \| \sketchop[\Phi](\cl P_{\solquant}) - \sketchop[\Phi]( \cl Q ) \|_2\\
	&\leq \lripc \| \bs a' - \Vec{z}_{\Phi} \|_2 + \lripc \| \Vec{z}_{\Phi} - \sketchop[\Phi]( \cl Q ) \|_2.
\end{split}
\end{equation*}
Using Lemma~\ref{lem:LRIPandLPD} then the optimality of~\eqref{eq:sketchmatching}, we get
\begin{equation*}
\begin{split}
	\ts \frac{\| \cl P_{\solquant} - \cl Q \|_{\cl R}}{\lripc} &\leq  2\sqrt{\epsilon} + \| \bs a - \Vec{z}_{\Phi} \|_2 + \| \Vec{z}_{\Phi} - \sketchop[\Phi]( \cl Q ) \|_2  \\
	&\leq 2\sqrt{\epsilon} +  2 \| \Vec{z}_{\Phi} - \sketchop[\Phi]( \cl Q ) \|_2.
\end{split}
\end{equation*}
We develop the second term with the triangle inequality,
$$\| \Vec{z}_{\Phi} - \sketchop[\Phi]( \cl Q ) \|_2 \leq \| \Vec{z}_{\Phi} - \sketchop[\Phi]( \cl P_0 ) \|_2 + \| \sketchop[\Phi]( \cl P_0 ) - \sketchop[\Phi]( \cl Q ) \|_2.$$
Gathering the results, and taking the infimum with respect to $\cl Q \in \modelset$, we obtain that $\| \cl P_{\solquant} - \cl P_0 \|_{\cl R} \leq \frac{\eta'}{2}$, with
\begin{equation*}
\ts \frac{\eta'}{2} = D(\disttrue, \modelset) + 2 \lripc \| \Vec{z}_{\Phi} - \sketchop[\Phi]( \disttrue ) \|_2 + 2\lripc\sqrt{\epsilon}.
\end{equation*}
Finally, we combine this with the risk metric definition~\eqref{eq:excessRiskMetric}, and since the map $\bs \theta \mapsto \cl P_{\bs \theta}$ satisfies the risk consistency property~\eqref{eq:risk-consistency}, we get
\begin{align*}
  &\cl R(\solquant;\disttrue) - \cl R(\soltruerisk;\disttrue)\\
  &= \; \cl R(\solquant;\disttrue) - \cl R(\solquant;\cl P_{\solquant}) + \cl R(\solquant;\cl P_{\solquant}) - \cl R(\soltruerisk;\cl P_{\solquant})\\
&\qquad + \cl R(\soltruerisk;\cl P_{\solquant}) - \cl R(\soltruerisk;\disttrue)\\
&\leq \ts \frac{\eta'}{2} + 0 + \frac{\eta'}{2} = \eta'.
\end{align*}
\end{proof}

For comparison, the excess risk guarantee of Prop.~\ref{prop:risk_thiswork} for the asymmetric sketch matching solution~\eqref{eq:asymsketchmatching} is thus exactly the same as the guarantee in Thm.~\ref{prop:risk_gribonval} for the symmetric solution~\eqref{eq:sketchmatching}, up to the additive term $4\sqrt{\epsilon}$ in the excess risk bound (\ie $\eta' = \eta + 4\sqrt{\epsilon}$), which expresses the ``mismatch'' between $\Psi$ and $\Phi$. Whenever the LPD holds with a reasonably small error $\epsilon$, we can thus expect that the asymmetric scheme (learning from the distorted sketch $\bs z_{\Psi,\ds}$) will perform almost as well as the symmetric one (learning from the reference sketch $\bs z_{\Phi,\ds}$).

Of course, it remains to show that the LPD actually holds in practice to complete this guarantee, which we tackle next.

\section{Proving the LPD for quantized CL}
\label{sec:main-quantized}

In this section, we first provide one possible strategy to prove the LPD. Under an additional assumption (which we first introduce and show to be met in practice for k-means and GMM), it is sufficient to prove a somewhat simpler ``signal-level'' version of the LPD instead (which we call signal-LPD, or \emph{sLPD} for short). We then apply this strategy to the specific case where $\Phi = \Phi_{\RFF}$ are the random Fourier features and $\Psi = \Psi_{f}$ are generic random periodic features~\eqref{eq:RPF}, \ie the same RFF but where $t \mapsto \exp(\im t)$ is replaced by a generic periodic function $f(t)$. This finally allows us to obtain, as a particular case, statistical learning guarantees for the quantized CL scheme introduced in~\cite{schellekens2018quantized}, \ie where the distorted map is the binarized RFF $\Psi_{q}$~\eqref{eq:qRFF}. To demonstrate the generic nature of our result, we also apply it to the modulo features $\Psi_{\smod}$.

\subsection{Assumption on the data domain}
\label{sec:assumpt-data-doma}

In practical machine learning applications, the data vectors $\bs x_i$ do not take \emph{any} possible value in $\bb R^d$---one can usually assume \emph{a priori} that they belong to a compact set $\Sigma \subset \bb R$. We extend this assumption to the \emph{probability distributions} involved in the compressive learning problem~\eqref{eq:asymsketchmatching}, \ie to the empirical distributions $\distemp \in \measuredset$ (see Assumption~\ref{as:empirical_set}) and the parametric distributions $\cl P_{\bs \theta} \in \modelset$.

\begin{assumption}
\label{as:domain}
	For a compact set $\Sigma \subset \Rbb^d$, the model set $\modelset$ and the empirical set $\measuredset$ are subsets of $\setdistributions_{\Sigma,\zeta} \subset \setdistributions$, the set of probability measures that are \emph{mostly} supported on $\Sigma$, \ie for some $0 \leq \zeta < 1$, $\modelset, \measuredset \subset \setdistributions_{\Sigma,\zeta}$ with
	\begin{equation}
		\label{eq:modelSet}
		\ts \setdistributions_{\Sigma,\zeta} := \{ \cl P \in \setdistributions : \cl P(\Sigma) = \Integr{\Sigma}{}{}{\cl P} \geq 1- \zeta \}.
	\end{equation}
\end{assumption}

In particular, for $\zeta = 0$ all $\cl P \in \setdistributions_{\Sigma,0}$ are ``almost surely'' supported on $\Sigma$, \ie $\supp(\cl P) \subset \Sigma$.

Assumption~\ref{as:domain} holds in practice. Consider the common case where there are known lower and upper bounds $\bs l, \bs u \in \bb R^d$ for the values that the learning samples $\Vec{x}_i$ can take (\eg due to physical constraints). This means that all the learning examples lie in a ``box'' $\Sigma_{\bs l, \bs u} \subset \bb R^d$:
$$\bs x_i \in \Sigma_{\bs l, \bs u} := \{\bs x \in \Rbb^d : \bs l \leq \bs x \leq \bs u\}.$$
Since all the examples of any considered dataset $\ds$ necessarily lie in that box, this directly implies that all the related empirical distributions $\distemp \in \measuredset$ satisfy $\distemp \in \setdistributions_{\Sigma_{\bs l, \bs u},0}$. 

The inclusion of the model set $\modelset$ can be reached from additional constraints on $\Theta$, imposed during the optimization procedure. In the simplest case, the \emph{k-means} task (Table~\ref{tab:tasks}, left), the optimal centroids obviously lie inside the data-enclosing box, hence it makes sense (as done in~\cite{keriven2016compressive}) to restrict the problem to $\bs c_k \in \Sigma_{\bs l, \bs u}$. These constraints can be encoded in $\Theta$, which in turn imply that $\modelset \subset \setdistributions_{\Sigma_{\bs l, \bs u},0}$. If the data-enclosing box $\Sigma_{\bs l, \bs u}$ is known, Assumption~\ref{as:domain} thus holds for k-means, with $\Sigma = \Sigma_{\bs l, \bs u}$ and $\zeta = 0$.

For the \emph{Gaussian mixture modeling} task (Table~\ref{tab:tasks}, right), one can similarly enforce that the Gaussian centers lie in the box $\bs \mu_k \in \Sigma_{\bs l, \bs u}$. Moreover, given that the data lie in a bounded domain, it is also reasonable to bound the variance of the Gaussian modes (the typical spread of a Gaussian mode should not be much larger than the box). This can be done by bounding the eigenvalues of the covariance matrices $\bs \Gamma_k$, \ie $\lambda_{\max}(\bs \Gamma_k) \leq S$ for a bound $S > 0$ to be set according to the size of the box. Assuming 
diagonal covariances, as commonly done CL for GMM~\cite{keriven2016GMMestimation}, the following lemma shows that $\modelset \subset \setdistributions_{\Sigma^{(\rho)},\zeta}$ for some slightly extended set $\Sigma^{(\rho)} \supset \Sigma_{\bs l, \bs u}$ and small $\zeta>0$.

\begin{lemma}[Assumption~\ref{as:domain} for GMM with box contraints]
\label{lem:GMMassumption}
	Let $\modelset$ be the model set of a GMM task with diagonal covariances $\bs \Gamma_k$ and the box constraints $\bs \mu_k \in \Sigma_{\bs l, \bs u}$ and $\lambda_{\max}(\bs \Gamma_k) \leq S$. Given $\rho > d$, we define the bounds of an ``extended'' box $\wt{\bs l} = \bs l - \rho S \bs{1}$ and $\wt{\bs u} = \bs u + \rho S \bs{1}$, with $\bs{1} = (1, 1, ..., 1) \in \bb R^d$. Then, for $\Sigma^{(\rho)} = \Sigma_{\wt{\bs l}, \wt{\bs u}}$, we have
	$$\ts \modelset \subset \setdistributions_{\Sigma^{(\rho)},\zeta},\quad\text{with}\ \zeta \lesssim e^{-\rho^2}.$$
\end{lemma}
\begin{proof}
  From \eqref{eq:modelSet}, we can set $\zeta = \sup\{\int_{\Sigma^{(\rho)c}} \ud \cl P: \cl P \in \cl G\}$ to ensure $\cl G \subset \cl M_{\Sigma^{(\rho)}, \zeta}$, \ie the maximal failure probability is reached when the largest amount of the GMMs probability mass lies outside $\Sigma^{(\rho)}$. For diagonal covariances, this is reached if we have both $\bs \Gamma_k = S \bs I_d$, \ie the Gaussian is maximally spread in each dimension, and, by symmetry, each Gaussian mode is located at a corner of the box $\Sigma_{\bs l, \bs u}$, \eg with $\bs \mu_k = \bs u$ for all $k$. Denoting by $\cl P^* \subset \cl G$ this GMM configuration, we easily show that, for $\Sigma^{(\rho)} = \Sigma_{\wt{\bs l}, \wt{\bs u}}$, \ie $\Sigma^{(\rho)} = \Sigma_{\bs l, \bs u} + \rho S \bb B_\infty^d$,  $\int_{\Sigma^{(\rho)}} \ud \cl P^*$ is bounded by $\phi^d(\rho)$, with $\phi$ the cumulative density function of a one-dimensional standard normal random variable. Using well-known Gaussian tail bounds~\cite{vershynin2018high}, we get $\zeta = 1 - \phi^d(\rho) \leq d \cdot \phi(\rho) \leq \frac{d}{\rho \sqrt{2\pi}} e^{-\rho^2/2}$, which decays exponentially fast in $\rho$; in particular $\zeta \lesssim e^{-\rho^2}$ when $\rho > d$.
\end{proof}
To wrap up, Assumption~\ref{as:domain} thus seems fairly reasonable for the GMM task as well. There is however still the issue of which value of $\rho$ one should pick in Lemma~\ref{lem:GMMassumption}; we'll come back to this nontrivial question at the end of this section.

\subsection{Reducing the LPD to the signal-level LPD}
Under Assumption~\ref{as:domain}, to have the LPD, it is sufficient to prove a simpler ``signal-LPD'' (sLPD) over $\Sigma$, defined as
\begin{equation}
\label{eq:signalLPD}
	\forall \bs x, \bs y \in \Sigma, \: |\langle\Psi(\bs x)- \Phi(\bs x),\Phi(\bs y) \rangle| \leq \epsilon_s.
\end{equation}
This is formalized by the following lemma.

\begin{lemma}[sLPD implies LPD]
\label{prop:easyGeneralLPD}
For a compact set $\Sigma$ and $0 \leq \zeta < 1$, assume $\modelset, \measuredset \subset  \setdistributions_{\Sigma,\zeta}$ and the sLPD~\eqref{eq:signalLPD} holds on $\Sigma$ with error $\epsilon_s$. Moreover, assume that the feature maps are bounded, $\sup_{\Vec{x} \in \bb R^d}\| \Phi(\Vec{x}) \|_2 \leq C_{\Phi}$ (similarly for $\Psi$).
Then, $\sketchop[\Psi]$ has the $\LPD(\epsilon;\measuredset,\modelset, \sketchop[\Phi])$ with error $\epsilon = \epsilon_s + c \zeta$, where $c \leq 2C_{\Phi}( C_{\Phi} +  C_{\Psi})$.
\end{lemma}
\begin{proof}
	Define the ``difference kernel''
	$$\wt{\kappa}(\bs x,\bs y) := \langle \Psi(\bs x) - \Phi(\bs x) ,\Phi(\bs y) \rangle,$$
	which is bounded by $|\wt{\kappa}(\bs x,\bs y)| \leq  C_{\Phi}( C_{\Phi} +  C_{\Psi}) =: C_{\wt{\kappa}}$ using Cauchy-Schwarz. Given any $\cl P,\cl Q \in \setdistributions_{\Sigma,\zeta} $, we have
	\begin{align*}
          &\ts | \langle\sketchop[\Psi] (\cl P) - \sketchop[\Phi](\cl P) , \sketchop[\Phi](\cl Q) \rangle | = \big| \expec{\bs x \sim \cl P}\expec{\bs y \sim \cl Q} \wt{\kappa}(\bs x,\bs y) \big|\\
          &\ts \leq \expec{\bs x \sim \cl P}\expec{\bs y \sim \cl Q} \big| \wt{\kappa}(\bs x,\bs y) \big|\\
        &= \ts I(\Sigma, \Sigma)+I(\Sigma^c, \Sigma)+I(\Sigma, \Sigma^c)+I(\Sigma^c, \Sigma^c)\\
	&\leq \epsilon_s (1-\zeta)^2 +  C_{\wt{\kappa}} (2\zeta(1-\zeta) + \zeta^2 ) \leq \epsilon_s + 2C_{\wt{\kappa}} \zeta,
	\end{align*}
        with $I(U,V) = \Integr{U}{}{\Integr{V}{}{\left|\wt{\kappa}(\bs x,\bs y) \right| \:}{\cl P(\Vec{x})}}{\cl Q(\Vec{y})}$, $U, V \subset \bb R^d$.
\end{proof}

Combining with Prop.~\ref{prop:risk_thiswork}, under Assumption~\ref{as:domain}, the sLPD~\eqref{eq:signalLPD} with error $\epsilon_s$ thus implies the excess risk is controlled, with an additive increase given by $4\sqrt{\epsilon_s + c\zeta}$ compared to Thm.~\ref{prop:risk_gribonval}.

\subsection{Proving the signal-LPD for random periodic features}

The tools developed up to now were purposefully as general as possible. We now focus on proving the LPD for the case where the reference feature map are random Fourier features~\eqref{eq:RFF} $\Phi_{\RFF}$, and the distorted feature map are random periodic features~\eqref{eq:RPF} $\Psi_f$; as a reminder, these are defined by replacing $\exp(\im t)$ in the usual RFF by a generic $2\pi$-periodic function $f(t)$, \ie $\ts \Psi_f(\bs x) := \frac{1}{\sqrt{m}} f(\bs\Omega^{\top}\bs x + \bs \xi)$. Quantized RFF~\eqref{eq:qRFF} $\Psi_{q}$ are a particular case of RPF, further explored in the next subsection.

To satisfy the sLPD, the RPF feature map must be scaled by $F_1$, the first Fourier Series (FS) coefficient of $f$~\cite{schellekens2018quantized,schellekens2020breaking}. We thus consider the \emph{renormalized features} $\overline{\Psi}_{f} := \frac{1}{F_1}\Psi_f$ during the ACL optimization procedure. This renormalization is not restrictive, it can for example be performed after the sketch has been computed, since $\bs z_{\overline{\Psi}_{f},\cl X} = \frac{1}{F_1} \bs z_{\Psi_{f},\cl X}$.

We can then prove our second main result: the LPD~\eqref{eq:LPD} holds---with high probability on the draw of $\bs\Omega$ and $\Vec{\xi}$---for the (normalized) RPF sketch operator. The proof is based on Lemma~\ref{prop:easyGeneralLPD} to reduce the LPD to the sLPD, and the latter is shown using a result of our previous work~\cite[Cor.1]{schellekens2020breaking}.

\begin{proposition}[LPD for normalized RPF]
\label{prop:slPDforRPF}
Let us consider a compact set $\Sigma$ with Kolmogorov $\nu$-entropy $\cl H_{\nu}(\Sigma) < \infty$, and 
  \begin{itemize}[leftmargin=*]
  \item the random Fourier features $\Phi_{\RFF}$ defined in~\eqref{eq:RFF} associated with the distribution $\Lambda$ (which generates the $m$ columns of $\bs\Omega$), with \emph{smoothness constant}
    \begin{equation}
      \label{eq:smoothness-constant}
      \ts C_{\Lambda} := \max_{\bs a \in \bb R^d, \|\bs a\|_2 = 1} \; \expec{\bs \omega \sim \Lambda} |\bs \omega^{\top} \bs a| < \infty,
    \end{equation}
      \item the related RPF $\overline{\Psi}_{f} := \frac{1}{F_1}\Psi_f$, with $\Psi_f$ defined in~\eqref{eq:RPF}, and $f$ a periodic, mean Lipschitz smooth function with constant $L^{\mu}_f < \infty$ (see Def.~\ref{def:meanSmoothness}), FS coefficients $\{F_k\}_{k\in \bb N}$, and $C_f := (1+\|f\|_{\infty}/|F_1|) < \infty$.
  \end{itemize}
  Given some $0 \leq \zeta < 1$ and $\epsilon_0>0$, assume that:
  \begin{itemize}[leftmargin=*]
  \item[] (i) $\modelset, \measuredset \subset  \setdistributions_{\Sigma,\zeta}$ (Assumption~\ref{as:domain});
  \item[] (ii) the sketch dimension $m$ satisfies
	$$m \geq 128  \cdot \epsilon_0^{-2} \cdot \cl H_{\epsilon_0/c_f}(\Sigma)$$
	with constant $c_f := 4 C_{\Lambda} (4 + L^{\mu}_f/|F_1|)$.
  \end{itemize}
	
	\noindent Then, with probability exceeding $1-3\exp(-m \epsilon_0^2/64)$ on the draw of $\bs\Omega$ and $\Vec{\xi}$, the normalized RPF sketch operator $\sketchop[\overline{\Psi}_{f}]$ has the LPD over $\measuredset$ with respect to $\sketchop[\Phi_{\RFF}]$, with error
	$$\epsilon= C_f \big( \epsilon_0 +2 \zeta \big) .$$
\end{proposition}
\begin{proof}
  Let us define the periodic function $\wh{f}(t) = \frac{1}{D}\big(\frac{1}{F_1} f(t) -  \exp(\im t)\big)$, with $D := \max(1,\|\frac{1}{F_1} f(\cdot) - \exp(\im \cdot) \|_{\infty})$. Since $\| \exp(\im \cdot) \|_{\infty} \leq 1$ and $\| \wh{f} \|_{\infty} \leq 1$, Corollary 1 from~\cite{schellekens2020breaking} ensures that, for any $\epsilon_0 > 0$, if $m \geq 128 \cdot \epsilon_0^{-2} \cdot \cl H_{\epsilon_0/c}(\Sigma)$, then, with probability at least $1-3\exp(-m \epsilon_0^2/64)$, we have for all $\bs x,\bs y \in \Sigma$
    \begin{equation}
      \label{eq:sLPD-tmp}
      \ts \frac{1}{m}\big| \langle \wh{f}(\bs\Omega^{\top}\Vec{x}+\Vec{\xi}), \exp(\im(\bs\Omega^{\top}\Vec{y}+\Vec{\xi})) \rangle  \big| \leq \epsilon_0,
    \end{equation}
     where the constant $c$ in the metric entropy radius is 
	$$c = 4 C_{\Lambda}(L^{\mu}_{\wh{f}} + L^{\mu}_{\exp(\im \cdot)} + 2\min(L^{\mu}_{\wh{f}}, L^{\mu}_{\exp(\im \cdot)})).$$
	By definition of $\wh{f}$,  $\Phi_{\RFF}$, and $\overline{\Psi}_f$, the bound \eqref{eq:sLPD-tmp} is equivalent to the sLPD property~\eqref{eq:signalLPD} with error $\epsilon_s = D \epsilon_0$.

	The mean Lipschitz constant of the relevant functions reads $L^{\mu}_{\exp(\im \cdot)} \leq 1$ and $L^{\mu}_{\wh{f}} \leq D^{-1}(\frac{1}{|F_1|}L^{\mu}_{f} + L^{\mu}_{\exp(\im \cdot)})$, and we apply the simplification $\min(L^{\mu}_{\wh{f}}, L^{\mu}_{\exp(\im \cdot)}) \leq L^{\mu}_{\exp(\im \cdot)}$. Since $D \geq 1$, we thus get that $c$ is upper bounded by
	$$\ts 4 C_{\Lambda}\big(3 + D^{-1}(1 + \frac{1}{|F_1|}L^{\mu}_{f} )\big) \leq c_f := 4 C_{\Lambda}(4 + \frac{1}{|F_1|}L^{\mu}_{f} ).$$
        
    We can now turn the sLPD \eqref{eq:sLPD-tmp} into the LPD (with a proper rescaling of the error) by applying Lemma~\ref{prop:easyGeneralLPD}. We note that $C_{\Phi_{\RFF}} \leq \| \exp(\im \cdot) \|_{\infty} = 1$ and $C_{\overline{\Psi}_{f}} \leq \|f / F_1\|_{\infty} = \|f\|_{\infty} / |F_1|$, which implies that $2 C_{\Phi_{\RFF}} (C_{\Phi_{\RFF}} + C_{\overline{\Psi}_{f}}) \leq 2( 1 +   \|f\|_{\infty} / |F_1|) = 2C_f$. Finally, since $D \leq (1 + \|f\|_{\infty}/F_1) = C_f$, Lemma~\ref{prop:easyGeneralLPD} shows that the desired LPD holds with error $C_f(\epsilon_0 + 2 \zeta)$.

\end{proof}

\subsection{ACL guarantees for quantized or modulo contributions}

To formulate our final guarantees and use Lemma~\ref{lem:GMMassumption}, we need this bound on the Kolmogorov entropy of $\Sigma_{\wt{\bs l}, \wt{\bs u}} \supset \Sigma_{\bs l, \bs u}$.
\begin{lemma}
\label{lem:entropyOfExtendedBox}
	In the notations of Lemma~\ref{lem:GMMassumption}, we have 
	\begin{equation}
          \label{eq:tilde-box-entropy}
		\ts \cl H_{\nu}(\Sigma_{\wt{\bs l}, \wt{\bs u}}) \leq d \log\big(1 + \frac{\sqrt d (2 \rho S + \|\bs u - \bs l\|_\infty)}{\nu}\big).
	\end{equation}
\end{lemma}
\begin{proof}
Since $\Sigma_{\wt{\bs l}, \wt{\bs u}} \subset \bs c + (\rho S + r) \bb B_{\infty}^d$, with $2 \bs c = \bs l + \bs u$ and $2 r = \|\bs u - \bs l \|_\infty$, the entropy of $\Sigma_{\wt{\bs l}, \wt{\bs u}}$ is bounded by the one of $\cl B := (\rho S + r) \bb B_{\infty}^d$.  From Lemma 4.10 in \cite{pisier1999volume}, given $\nu' > 0$, there exists a $\nu'$-covering $\cl S$ of $\cl B$ in the $\ell_\infty$-metric---\ie for all $\bs x \in \cl B$, there is one $\bs q \in \cl S$ such that $\|\bs x - \bs q\|_\infty \leq \nu'$---where $|\cl S|\leq(1 + \frac{2(\rho S + r)}{\nu'})^d$.  Since $\|\bs x - \bs q\|_\infty \geq \|\bs x - \bs q\|_2/\sqrt{d}$, $\cl S$ is also a ($\ell_2$) covering of $\cl B$ with radius $\nu' \sqrt{d}$. Taking $\nu'=\nu/\sqrt{d}$ shows that the Kolmogorov $\nu$ entropy of $\cl B$, and thus that of $\Sigma_{\wt{\bs l}, \wt{\bs u}}$, is bounded as in~\eqref{eq:tilde-box-entropy}. 
\end{proof}

We can finally combine all our results together to obtain statistical learning guarantees (excess risk bounds) for the ACL problem with RPF sketches (and in particular, for quantized or modulo RFF), when solving the tasks of k-means and GMM specifically, under the assumption that the data is constrained in the box $\Sigma_{\bs l, \bs u}$.

\begin{corollary}
\label{cor:final}
In the notations of Prop.~\ref{prop:slPDforRPF}, consider $\solquant$ the solution to the ACL problem~\eqref{eq:asymsketchmatching}, where the sketch is obtained by the normalized random periodic features $\overline{\Psi}_f$, and the reference sketch map is $\Phi_{\RFF}$.

  We assume that, \emph{(i)} all the data samples lie in a box $\Sigma_{\bs l, \bs u}$, and the known upper and lower bounds $\bs l$ and $\bs u$ are used to restrict the optimization procedure (for GMM, a constant $S$ is also used to upper bound the variance of the modes), and \emph{(ii)}, the LRIP holds for $\Phi_{\RFF}$ on the chosen task with constant $\lripc$.
  
  Then, given $\epsilon_0>0$, we have the following guarantees with probability at least $1-3\exp(-m\epsilon_0^2/64)$: 
  \begin{itemize}[leftmargin=*]
		\item{} [for \textbf{k-means}] If the sketch size is at least 
                  $$\ts m \geq 128 \cdot \epsilon_0^{-2} \cdot d \log\big(1 + \frac{c_f \sqrt d \|\bs u - \bs l\|_\infty}{\epsilon_0}\big),$$
		then, for $C_{\rm km} := 4 \lripc \sqrt{C_f}$, the excess risk is bounded by
		$$\ts \cl R(\solsketch';\disttrue) - \cl R(\soltruerisk;\disttrue) \leq \eta + C_{\rm km} \sqrt{\epsilon_0},$$
		with $\eta$ the excess risk in symmetric CL~\eqref{eq:excessrisk_gribonval}. 
		\item{} [for \textbf{GMM}] For some $\rho > d$, if the sketch size is at least, 
		\begin{equation*}
		\begin{split}
			\ts m \geq 128 \cdot \epsilon_0^{-2} \cdot d \log\big(1 + \frac{c_f \sqrt d (2\rho S + \|\bs u - \bs l\|_\infty)}{\epsilon_0}\big),
		\end{split}
		\end{equation*}
		then the excess risk is bounded by
		$$ \cl R(\solsketch';\disttrue) - \cl R(\soltruerisk;\disttrue) \leq \eta + C_{\rm km} \sqrt{\epsilon_0 } + C_{\rm gmm} \cdot  e^{-\rho^2/4}, $$
                with $C_{\rm gmm} := 4 \lripc \big( \frac{2 C_f^2}{\pi} \big)^{\frac{1}{4}}$, and $\eta$ and $C_{\rm km}$ as above.
	\end{itemize}
\end{corollary}
\begin{proof}
The proof consists in applying Prop.~\ref{prop:risk_thiswork} and Prop.~\ref{prop:slPDforRPF}, combined with Lem.~\ref{lem:GMMassumption} and Lem.~\ref{lem:entropyOfExtendedBox}, the entropy of $\Sigma_{\bs l, \bs u}$ being found by setting $\rho = 0$ in \eqref{eq:tilde-box-entropy}. 
\end{proof}

A few remarks can be made about this corollary. First, we rely on the fact that $\Phi_{\RFF}$ satisfies the LRIP; actually, ensuring that this holds (with high probability on the draw of $\bs\Omega$) imposes additional constraints on $m$. They depend on the considered task and the complexity of the related model set $\cl G$; for example, for a GMM with $K$ modes in $\bb R^d$, we should have $m = \Omega(K d)$, up to some additional factors and restrictions on $\cl G$ (see~\cite[Sec. 5.5]{gribonval2020sketching} and \cite{gribonval2020statistical}). Second, the choice of the parameter $\rho$ necessitates solving a trade-off: increasing $\rho$ decreases the excess risk bound (excess risk proportional to $e^{-\rho^2/4}$), but at the cost of logarithmically increasing the required sketch size $m$.

Finally, Cor.~\ref{cor:final} allows us to determine which between quantized or modulo RPF requires more measurements. Indeed, we compute in App.~\ref{sec: comp-mod-q-constant} that $C_q = 1+ \frac{\pi}{2\sqrt 2}<C_{\smod} = 1+ \frac{\sqrt 5 \pi}{4}$ and $c_q = 24 C_\Lambda < c_{\smod} = (24 + 2\sqrt 2) C_\Lambda$. Therefore, both for k-means and GMM, the sample complexities and the excess risk bounds of Cor.~\ref{cor:final} shows that, while requiring a higher number of measurements, ACL with modulo features gets a higher bound on the excess risk compared to that of a quantized sketch. Thus, the sketch size must be further increased (to allow a smaller $\epsilon_0$) when using modulo sketches in order to meet quantized sketch performances. We observe this effect experimentally in the following section.

\section{Experiments}
\label{sec:experiments}

We further validate the ACL approach through three practical numerical experiments. Going beyond the preliminary results in~\cite{schellekens2018quantized} for compressive k-means with quantized contributions $\Psi_q$, we also try out the ACL scheme with the modulo feature map $\Psi_{\smod}$, for the GMM task, explore the impact of the dataset size, and apply it to large-scale settings such as audio feature extraction for event recognition.

Remember that the cost $\cl C_{\Phi}(\bs \theta ; \bs z)$ is not convex, and that the symmetric and asymmetric compressive learning problems, described respectively by~\eqref{eq:sketchmatching} and~\eqref{eq:asymsketchmatching}, cannot be solved exactly in practice. To approximate those solutions, we thus mainly use the \texttt{CLOMP} greedy algorithm~\cite{keriven2016GMMestimation}; the exception is the last experiment which uses the Gaussian splitting algorithm (algorithm 2 in~\cite{keriven2016GMMestimation}) for better performances when the number of Gaussians $K$ is large. Regarding the implementation, we use \texttt{pycle}~\cite{pycle}, a general-purpose compressive learning toolbox in Python.

\emph{Remark}: Since the cost $\cl C_{\Phi}(\bs \theta; \bs z)$ behaves in the same manner with respect to $\bs \theta$ in the ACL case ($\bs z = \bs z_{\Psi,\ds}$) as in the usual symmectric CL case ($\bs z = \bs z_{\Phi,\ds}$), we can use the exact same algorithms in both scenarii.

We first consider a controlled environment, where we generate a synthetic dataset $\ds$ according to a known ``ground-truth'' Gaussian mixture model $\disttrue$.
To evaluate the quality of a CL solution $\solsketch$, we use the \emph{empirical excess risk}, 
\begin{equation}
	\label{eq:emp_excess_risk}
	\Delta \cl R(\solsketch) := \cl R(\solsketch;\distemp) - \cl R(\solerm;\distemp),
\end{equation}
where the empirical risk minimizer $\solerm$ is estimated by keeping the best out of several independent trials of traditional ML algorithms operating on the full dataset: the \texttt{k-means++} algorithm~\cite{Lloyd1982kmeans,Arthur2007kmeans++} for k-means clustering, and Expectation-Maximization~\cite{moon1996expectation} for GMM. 
Referring to Table~\ref{tab:tasks}, for k-means the empirical excess risk corresponds to the excess SSE~\eqref{eq:SSE}, $\Delta \cl R(\solsketch) = \SSE(\solsketch; \ds) - \SSE(\solerm; \ds)$, while for GMM, $\Delta \cl R(\solsketch) = \LL(\solerm; \ds) - \LL(\solsketch; \ds)$ is the excess negative log-likelihood~\eqref{eq:loglikelihood}. Because the performances depend on the random draw of $\bs \Omega$ and $\bs \xi$, we perform several independent trials of (A)CL and report the median performance.

The numerical value of the excess risk, while relevant to our theoretical guarantees, is not always easy to interpret. Therefore, another metric we use to assess the quality of solutions $\solsketch$ is the \emph{success rate}: the average number of ``successes'' obtained over all trials. For our purposes, we arbitrarily define the ``success'' of solution $\solsketch$ as follows: when we solve k-means, $\solsketch$ succeeds if $\SSE(\solsketch; \ds) \leq 1.2 \times \SSE(\solerm; \ds)$; when we solve GMM, $\solsketch$ succeeds if $\LL(\solsketch; \ds) \geq \frac{\LL(\solerm; \ds)}{1.2}$ (where we ensure $\LL(\solerm; \ds) > 0$).

\subsection*{Experiment 1: synthetic data, quantized/modulo features}
For this first experiment, $\cl P_0$ is a mixture of $K = 10$ Gaussian modes in dimension $d = 5$, from which we draw a dataset $\ds$ of $n = 10^5$ samples. We then sketch this dataset, using the standard random Fourier features $\Phi_{\RFF}$, but also the quantized RFF $\Psi_{q}$ and the modulo feature map $\Psi_{\smod}$, and solve both k-means and GMM from those sketches\footnote{The fully symmetric CL case, where $\Phi_{\RFF}$ is used for sketching, does not require the dither $\bs \xi$, so we impose $\bs \xi = \bs 0$ in that case. In the asymmetric case, recall we moreover perform a normalization $\bs z_{\overline{\Psi}_{f},\cl X} = \frac{1}{F_1} \bs z_{\Psi_{f},\cl X}$ before learning.}. We draw a varying amount $m$ of random frequencies $\bs \omega_j \distiid \Lambda$ from $\Lambda$ given by the ``Folded Gaussian'' heuristic described in~\cite{keriven2016GMMestimation}, with scale $\sigma^2 = \frac{1}{10d}$ for k-means and  $\sigma^2 = \frac{1}{100d}$ for GMM. Compared to a Gaussian distribution, this folded variant improves the sampling of low frequencies.

\begin{figure}
	\centering
	\subfloat[k-means]{
		\includegraphics[height=165px]{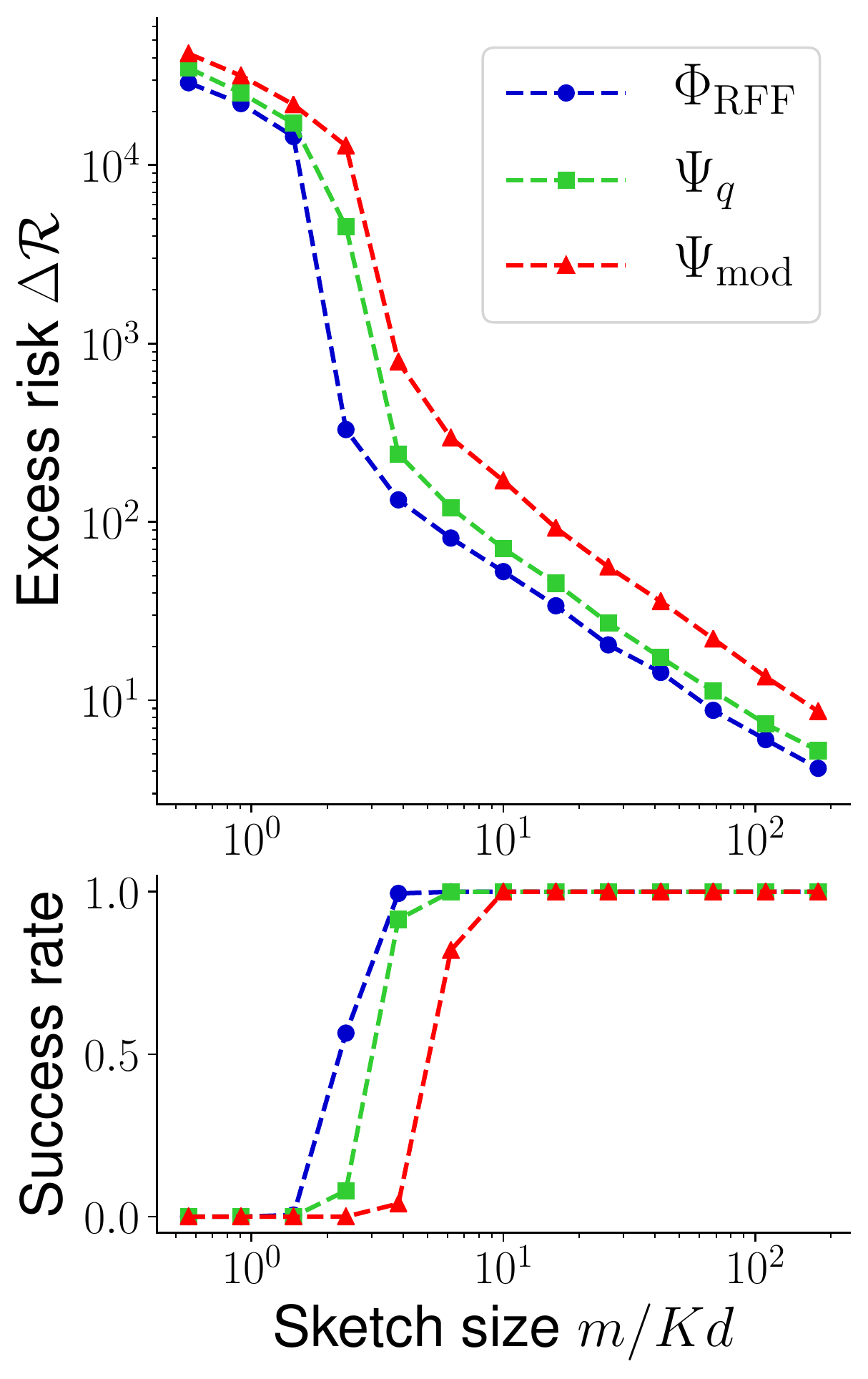}
		\label{fig:synth_0_kmeans}}
	\hspace{4px}
	\subfloat[GMM]{
		\includegraphics[height=165px]{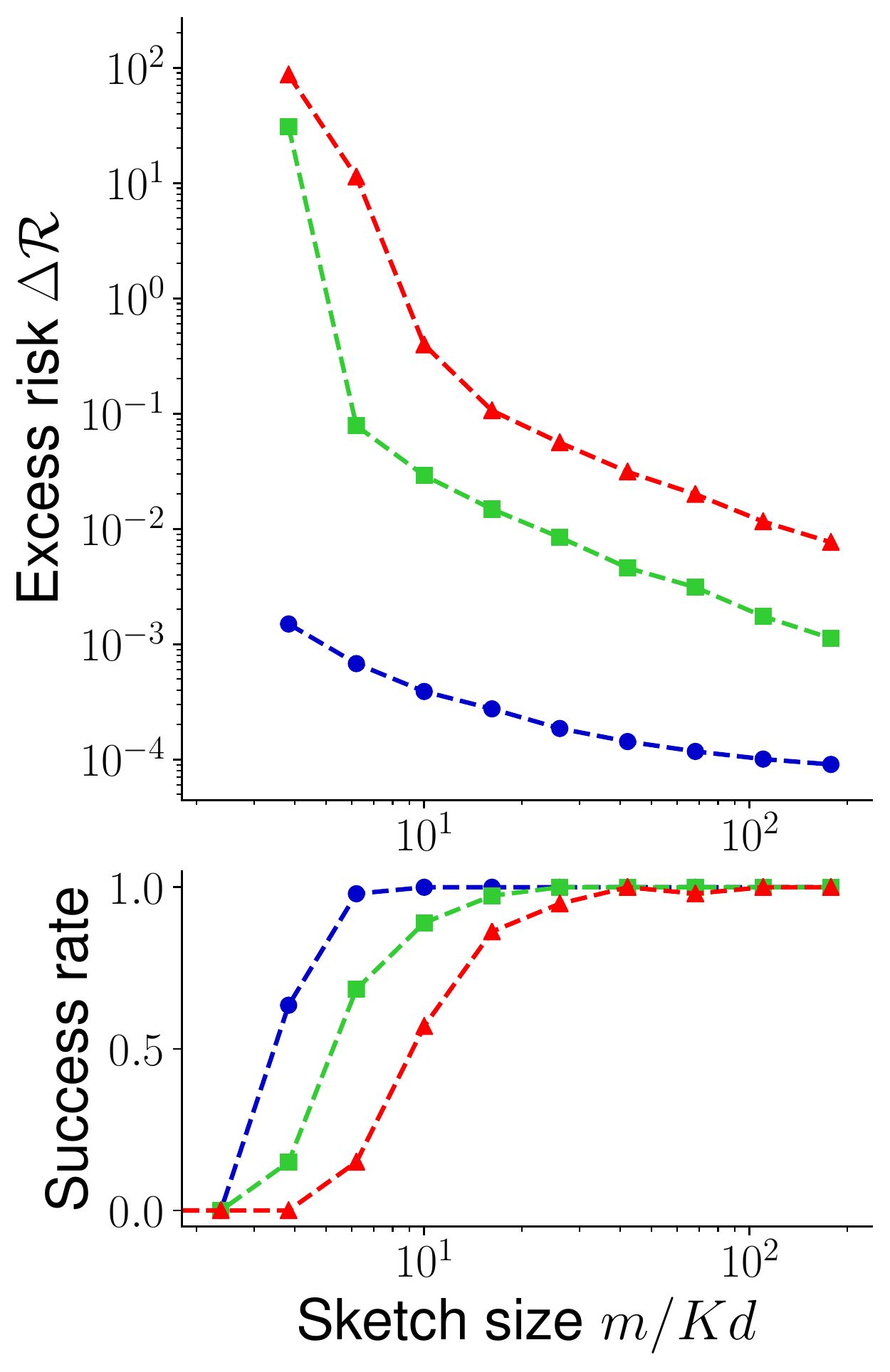}
		\label{fig:synth_0_gmm}}
	\caption{\emph{Top}: empirical excess risk~\eqref{eq:emp_excess_risk} for the k-means (\emph{left}) and GMM (\emph{right}) tasks, as a function of sketch size, obtained by the following (A)CL strategies: usual symmetric CL with $\Psi = \Phi = \Phi_{\RFF}$ (blue circles), asymmetric CL with quantized sketch contributions $\Psi = \Psi_{q}$ (green squares), and modulo sketch contributions $\Psi = \Psi_{\smod}$ (red triangles). Each data point is the median out of 25 or more independent trials. \emph{Bottom:} the associated success rate.}
	\label{fig:synth_0}
\end{figure}

The results are shown Fig.~\ref{fig:synth_0}. From Fig.~\ref{fig:synth_0_kmeans}, as we already observed in~\cite{schellekens2018quantized}, in the case of k-means one can use quantized sketch contributions with only a minor performance decrease (or, equivalently, a slight increase of the sketch size reaches the same performances). Moreover, ACL with modulo sketch contributions is also successful, but the sketch size (to reach a given performance level) must be larger than in the quantized case. This is indeed what would be expected from our theoretical results, as explained at the end of Sec.~\ref{sec:main-quantized}.

From Fig.~\ref{fig:synth_0_gmm}, we can observe that quantized or modulo ACL is also applicable to the task of GMM. As suggested by Cor.~\ref{cor:final}, the required sketch size (to reach equivalent performances) increases more for this task than for k-means.


\subsection*{Experiment 2: synthetic data, varying dataset size}
Next, we study the role of the dataset size in the ACL scheme (focusing on quantized ACL). Unless explicitly mentioned below, all parameters are identical to the previous experiment.
We first generate a ``full-size'' dataset $\wt{\ds}$ (with size $\wt{n} = 10^7$) from the previous GMM $\cl P_0$. For each trial, we use a smaller dataset $\ds$ for compressive learning obtained by picking uniformly at random a subset of $n$ samples (without replacement) in $\wt{\ds}$. Here, the empirical excess risk $\Delta \risk$ is evaluated using the full dataset $\wt{\ds}$ (the ERM minimizer $\solerm$ being also learned on this full dataset).

\begin{figure}
	\centering
	
	\subfloat[k-means]{
		\includegraphics[height=135px]{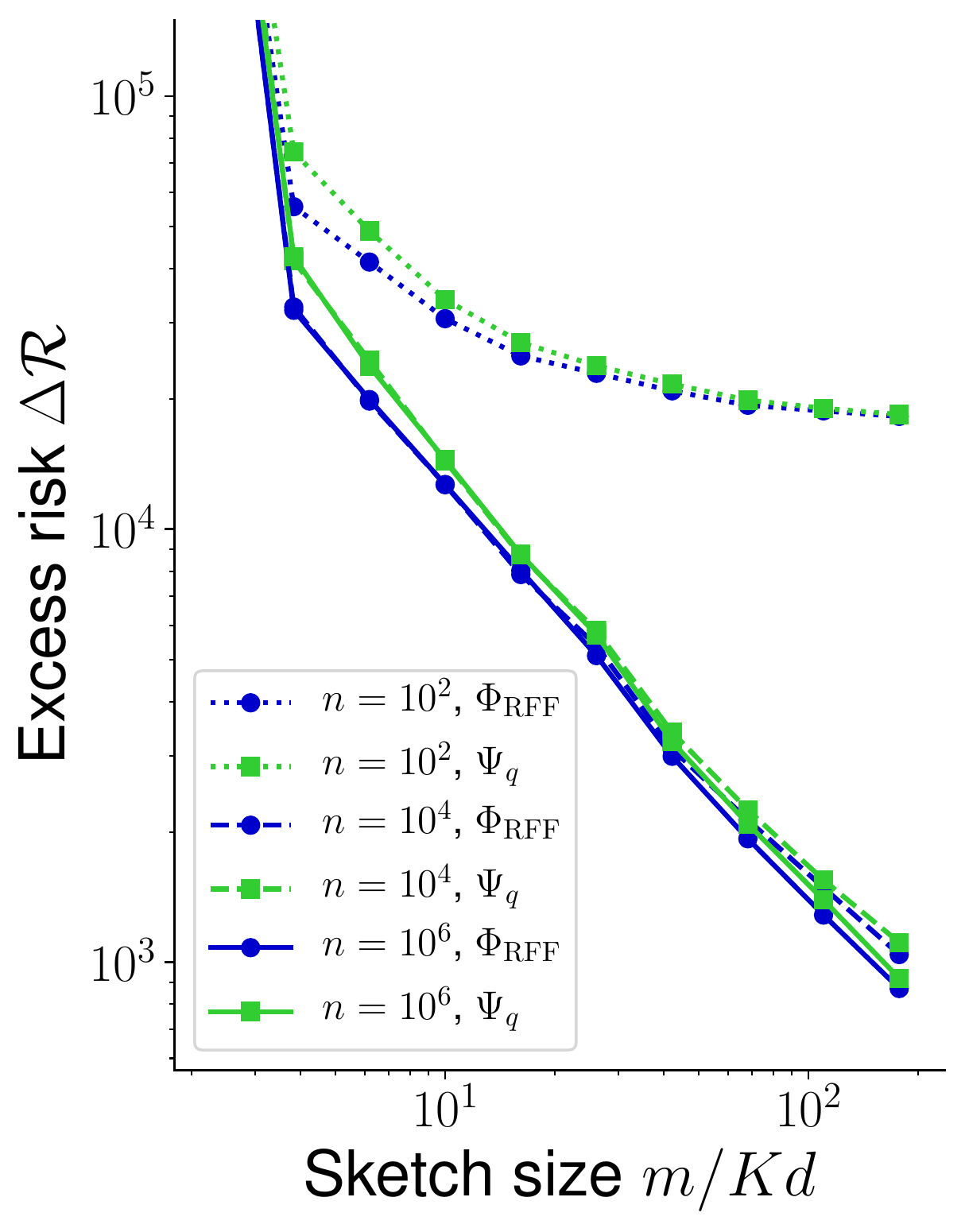}
		\label{fig:synth_1_kmeans}}
	\hspace{4px}
	\subfloat[GMM]{
		\includegraphics[height=135px]{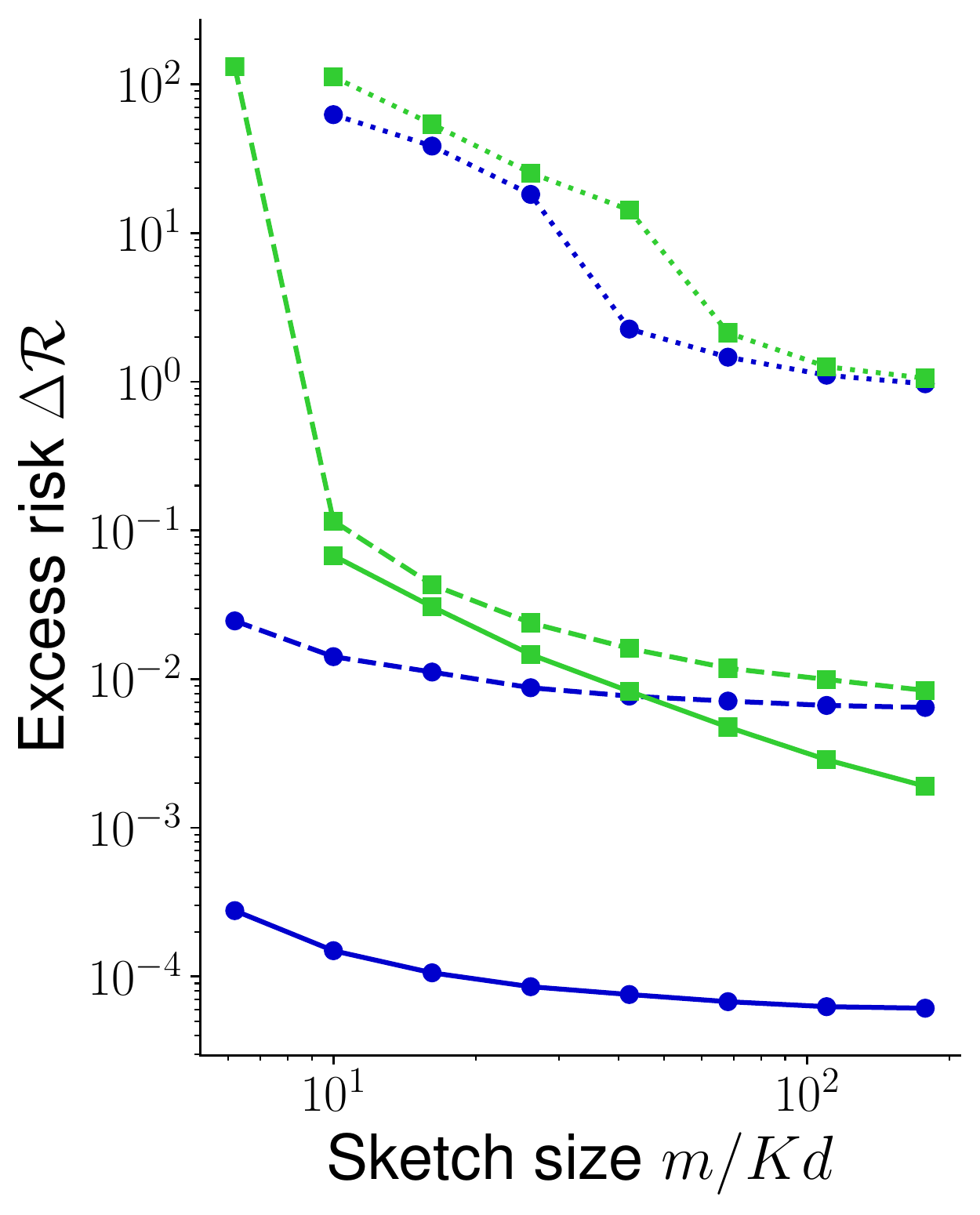}
		\label{fig:synth_1_gmm}}
	\caption{Empirical excess on a synthetic dataset of total size $\wt{n} = 10^7$ of solutions obtained through ACL where one sketches a subset of the dataset with varying sizes $n$ (dotted, dashed, plain lines for $n = 10^2, 10^4, 10^6$ respectively), using the full-precision sketch (blue circles) or the quantized sketch (green squares), as a function of the sketch size $m$. }
	\label{fig:synth_1}
\end{figure}

The results are shown Fig.~\ref{fig:synth_1}. As can be observed from Fig.~\ref{fig:synth_1_kmeans}, having a larger amount of samples $n$ improves the performances (at constant sketch size $m$), as could be expected. In terms of excess risk guarantees, this can be related to the sampling term $\lripc \| \sketchop[\Phi](\disttrue) - \sketchop[\Phi](\distemp) \|_2$ in~\eqref{eq:excessrisk_asym}. When $m$ increases, the excess risk quickly saturates on smaller datasets, but in the two larger-size datasets ($n = 10^4$ and $10^6$) this is not the case. An interesting phenomenon can be observed by looking (very) closely at those last two curves: there is a ``crossing'' between $m/Kd = 10$ and $m/Kd = 100$. When the sketch size is small, the curves are grouped by sketching feature map $\Psi$ (\ie by color); the dominant effect on the excess risk is whether the quantized sketch is used or not, regardless of the dataset size, which in our theoretical results can be associated with the LPD error term $\epsilon$. But as the sketch size increases, the curves are grouped by dataset size instead (\ie the plain and dashed curves go together); the dominant effect is now the sampling error. This can be explained by our theory through the fact that the LPD error $\epsilon$ decreases with\footnote{Note that the LRIP constant $\lripc$ also should decrease with $m$, but this same constant appears in all the terms of the excess risk, and this effect should thus impact all the curves in the same way.} $m$. Similar conclusions can be drawn for GMM modeling (Fig.~\ref{fig:synth_1_gmm}), albeit with a more significant impact related to the dataset size, which makes the crossing described before easier to observe.


\subsection*{Experiment 3: real data, audio classification task}
As large-scale proof-of-concept, we tackle an audio event classification task, where (A)CL is used to alleviate the computational cost of learning a GMM in a feature extraction phase. Note that our goal is not to propose a particularly competitive audio classification scheme, but to compare the ACL strategy to symmetric CL on large-scale, realistic data.

Our scheme follows the ``alpha features'' strategy described in~\cite{kumar2016features}. We use the ESC-50 dataset~\cite{piczak2015esc}, which contains $J = 2500$ audio clips lasting $5s$, each associated to one of $C = 50$ classes (\eg animals, water sounds, urban noises). We assume that the $J$ audio clips $\bs s^{(j)}$ are distributed across a sensor network, which perform local preprocessing as follows. For each audio clip $\bs s^{(j)}$, we extract\footnote{The MFCC extraction used the \texttt{librosa}~\cite{librosa} package. The subsequent SVM model is trained with \texttt{scikit-learn}~\cite{scikit-learn}.} Mel Frequency Cepstral Coefficients (MFCC), using $d=10$ frequency bands, $30$\,ms-long time intervals, and $15$\,ms-long hops. We then take the (distorted) features $\Psi(\bs x^{(j)}_i)$ (with $\Psi$ to be specified) of each of the $N = \lceil\frac{5\,{\rm s}}{15\,{\rm ms}}\rceil = 334$ resulting MFCC vectors $\bs x^{(j)}_i \in \Rbb^d$. Gathering features from the $J$ audio clips, the  $n = JN$ resulting contributions $\Psi(\bs x^{(j)}_i)$, each encoded by $b$ bits (more on this below), are aggregated, by a central server, into one sketch vector $\bs z_{\ds,\Psi}$. A GMM of $K = 32$ modes, describing the distribution of all clips in the frequency domain, is then extracted from this sketch by (A)CL. Each audioclip $\bs s^{(j)}$ can then by summarized by its ``alpha features'' $\bs \alpha^{(j)} \in \bb R^K$, defined by the average soft assignement of that clip's MFCC vectors to each of the $K$ Gaussian modes, \ie
$$\alpha^{(j)}_k :=  \ts \frac{1}{N} \sum_{i=1}^N  \frac{w_k p_{\cl N}(\bs x_i^{(j)}; \bs \mu_k, \bs \Gamma_k)}{\sum_{k=1}^K w_k p_{\cl N}(\bs x_i^{(j)}; \bs \mu_k, \bs \Gamma_k)}.$$
Finally, a SVM is learned on the alpha features to classify the $J$ audio samples; see~\cite{kumar2016features} for additional details.

We train this classification scheme and evaluate it on a separate test set ($20\%$ of the full dataset), for various values of the sketch dimension $m$. Assuming a scenario where minimizing the transmission cost is crucial, we compare the performances at given values of the number of bits $b$ sent per "message", \ie per featurized MFCC vector $\Psi(\bs x_i)$. For usual full-precision RFF $\Psi = \Phi_{\RFF}$, we assume the real and imaginary numbers are encoded by $64$ bits, \ie $b = 128m$. For the quantized RFF $\Psi = \Psi_{q}$ we have, by construction, $b = 2m$. As baseline, we also report the accuracy when Expectation-Maximization (EM) is performed on the whole dataset to learn the GMM (``un-compressed learning'').

\begin{figure}
	\centering
	\includegraphics[width=0.8\linewidth]{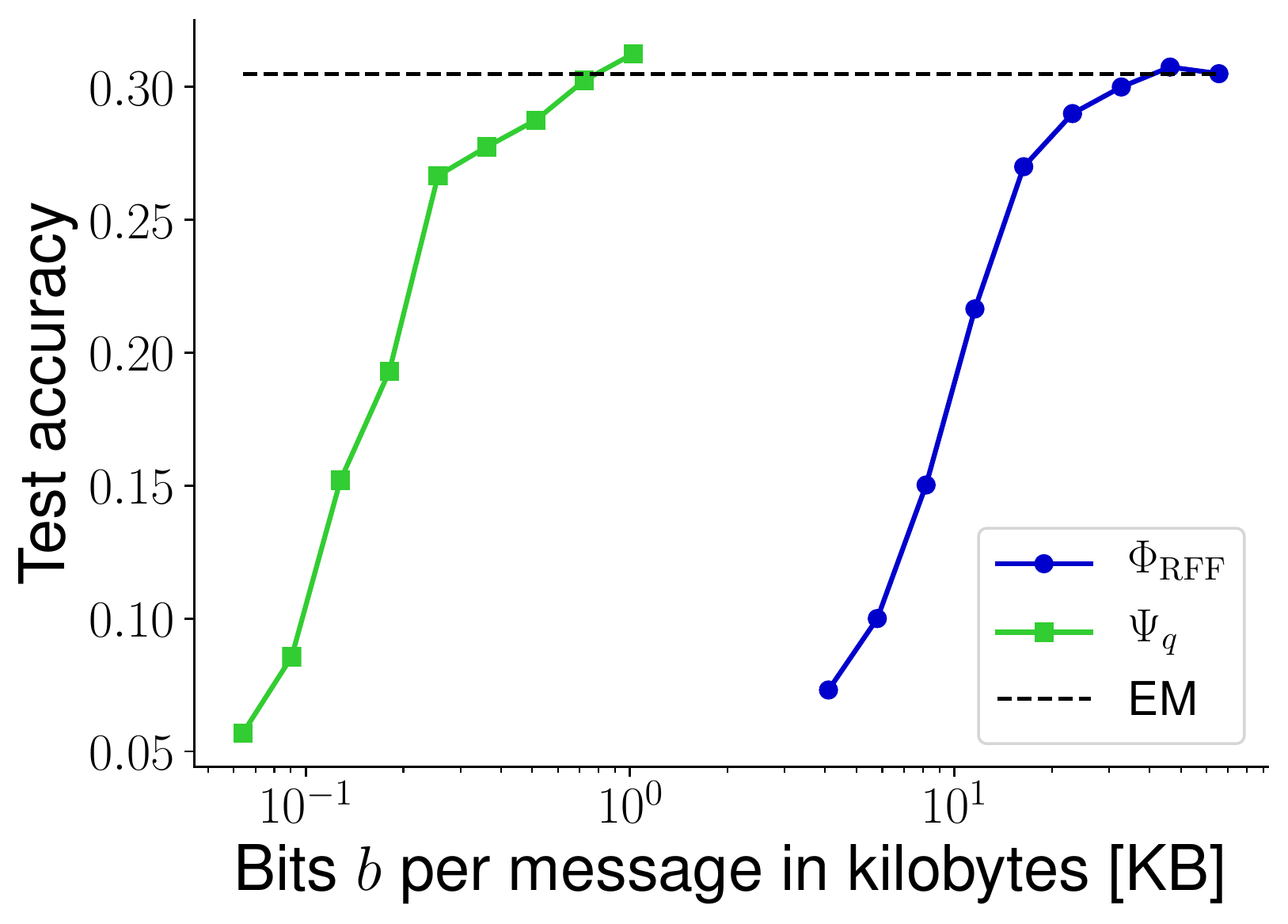}
	\caption{Test accuracy of the GMM-based audio classification procedure described in~\cite{kumar2016features}, as versus number of bits per contribution, for full-precision (blue) or quantized (green) sketch contribution, and with plain Expectation-Maximization (dashed black) on the ESC-50 dataset. }
	\label{fig:real_1}
\end{figure}

The results are shown Fig.~\ref{fig:real_1}. In this specific scenario, the ACL scheme with quantized contributions is particularly advantageous over the full-precision symmetric CL scheme, as \emph{the same performances can be reached with a bitrate reduction by a factor of at least $30$}. Moreover, when the sketch size is large enough, both compressive learning approaches are competitive with the EM baseline, which requires several passes over the entire database $\ds$. Of course, this doesn't mean CL is necessarily the best candidate for the scenario described here, as other approaches using \eg distributed learning could be considered; recall that we focus here on the comparison between symmetric and asymmetric CL.

\section{Conclusion}
\label{sec:conclusion}

As a first main contribution, we defined the asymmetric compressive learning (ACL) scheme and formally established excess risk bounds for it. This was achieved by introducing a specific LPD property---telling us ``how far'' a distorted feature map is from an undistorted one---that combines with the classical LRIP of compressive learning to explain the ACL performances. Our second key contribution was to apply this result (\ie proving the LPD property) to the specific case of quantized (and modulo) sketch contributions. We then further validated those results with numerical simulations.

However, our contribution focused on deriving an excess risk bound without particular care for its tightness. In particular, using Lemma~\ref{prop:easyGeneralLPD} to prove the LPD requires to have a sketch size $m$ scaling with the complexity (Kolmogorov entropy) of a signal set $\Sigma$, which is not a required ingredient in previous (symmetric) CL guarantees~\cite{gribonval2017compressiveStatisticalLearning}; this could be suboptimal if $\Sigma$ is large (in the Kolmogorov entropy sense); ideally, our results should depend on the complexity of the model set $\modelset$ rather than the signal space $\Sigma$.

Moreover, just as the existing CL guarantees, the excess risk bound is not readily exploitable in practice. For instance, the LRIP constant and the Kolmogorov entropy are hard to pin down accurately, and involve solving non-trivial trade-offs to be interpreted properly (\eg between the sketch size $m$, the probability of failures, the error contributions $\epsilon$). As a last caveat, let us recall that the current compressive learning algorithms used in practice are heuristics that do not have convergence guarantees. Future work is thus needed to bridge the gap between the theoretical guarantees and empirical performances of compressive learning---symmetric or not.

\appendices

\section{Quantized and modulo sketch constants}
\label{sec: comp-mod-q-constant}

We compute here the mean Lipschitz constant $L^{\mu}_{f}$, and the constants $C_f$ and $c_f$ defined in Prop.~\ref{prop:slPDforRPF},  when $f(\cdot)$ is either the universal quantization operation $q(\cdot)$ defined in (\ref{eq:unfi-quant-def}), or the (complex) modulo $\fmod(\cdot)$ in (\ref{eq:mod-f-def}).

In the case of $q$, its first FS coefficient is $Q_1 = \frac{4}{\pi}$, $\|q\|_{\infty} = \sqrt{2}$, and $L^{\mu}_{q} = \frac{8}{\pi}$ from~\cite[Prop.~6]{schellekens2020breaking}. This gives $C_q = (1+ \|q\|_\infty/|Q_1|) = 1+\frac{\pi}{2\sqrt{2}}$ and $c_q = 24C_{\Lambda}$.

Regarding the modulo function $\fmod$, since $\fmod_{2\pi}(t) = \sum_{k\neq0} \frac{\im}{\pi k} e^{\im k t}$, its first FS coefficient is $M_{1} = \frac{2\im}{\pi}$, and $\|\fmod\|_{\infty} = (1 + (\frac{1}{2})^2)^{1/2} = (\frac{5}{4})^{1/2}$, so that  $C_{\smod} = 1 + \sqrt{5}\pi/4$. Moreover, we have $L^{\mu}_{\smod} = \frac{4+\sqrt{2}}{\pi}$ since the integral in~(\ref{eq:meanSmooth}), \ie $ I_{\delta} := \int_{0}^{2\pi} \sup_{|r| \leq \delta} |\fmod(t+r)-\fmod(t)| \ud t$, 
can be upper bounded as (for $\delta \leq \frac{\pi}{4}$ the equality is reached)
\begin{equation*}
	\ts I_{\delta} \leq (2\pi - 4 \delta)\cdot\frac{\delta \sqrt{2}}{\pi} + 2 \cdot 2 \delta \cdot \frac{1}{2}(2 + 2 \sqrt{(1-\frac{\delta}{2\pi})^2 + \frac{\delta^2}{4\pi^2}}).
\end{equation*}
Therefore, since $(1-\frac{\delta}{2\pi})^2 + \frac{\delta^2}{4\pi^2} \leq 1$ (with equality if $\delta \in \{0,2\pi\}$), $\frac{I_{\delta}}{\delta} \leq g(\delta) := 8 - 6 \sqrt{2} + 8 \sqrt{2} (1 - \frac{\delta}{2\pi})$. This function $g$ reaches its maximum in $\delta=0$, where the equality holds, which means that $L^{\mu}_{\smod} = \sup_{0<\delta\leq\pi} \frac{I_{\delta}}{2 \pi \delta} = \frac{4+\sqrt{2}}{\pi}$, so that  $c_{\smod} = (24+2\sqrt{2})C_{\Lambda} > c_{q}$.

\newpage
\bibliographystyle{unsrt} 

\end{document}